%% file: main.tex
\documentclass{article} 

\usepackage[table]{xcolor}  
\usepackage{iclr2026_conference,times}

\iclrfinalcopy

\input{math_commands.tex}

\usepackage{hyperref}
\usepackage{url}

\usepackage{amsmath}
\usepackage{amssymb}
\usepackage{mathtools}
\usepackage{amsthm}

\usepackage{microtype}
\usepackage{graphicx}
\usepackage{booktabs} 

\usepackage{algorithm}
\usepackage{algpseudocode}
\usepackage{graphicx} 
\usepackage{todonotes}
\usepackage{bm}

\usepackage{xcolor}
\usepackage{dsfont}
\usepackage{nicefrac}
\usepackage{tcolorbox}
\usepackage{tabularx}
\usepackage{pifont}
\usepackage[inline]{enumitem}
\usepackage{multirow}
\usepackage{caption}
\usepackage{wrapfig}
\usepackage[export]{adjustbox}
\usepackage[list=true]{subcaption}
\usepackage{colortbl}
\definecolor{lightgray}{RGB}{242, 242, 242}
\usepackage{float}
\usepackage[capitalize,noabbrev]{cleveref}
\crefname{assumption}{Assumption}{Assumptions}
\crefname{algorithm}{Algorithm}{Algorithms}

\usepackage{longtable}
\usepackage[title]{appendix}

\usepackage{tcolorbox}
\usepackage{xspace}
\DeclareRobustCommand{\ie}{i.e.,\@\xspace}
  
\DeclareRobustCommand{\eg}{e.g.,\@\xspace}
\DeclareRobustCommand{\wrt}{w.r.t.\@\xspace}
\usepackage[capitalize]{cleveref}
\usepackage{tikz}
\usepackage{tikz-cd}
\usetikzlibrary{shapes,backgrounds}
\usepackage{array}
\usepackage{enumitem}

\usepackage{wrapfig}
\usepackage[export]{adjustbox}
\usepackage[list=true]{subcaption}
\usepackage[normalem]{ulem}
\usepackage{rotating}   
\usepackage{cleveref}
\usepackage{adjustbox}
\usepackage{twoopt}

\usepackage[T1]{fontenc}
\usepackage[utf8]{inputenc}
\usepackage{amsmath,amssymb}
\usepackage{booktabs}      
\usepackage{multirow}        

\usepackage{acronym}		

\usepackage{amsthm}
\usepackage{thmtools, thm-restate}

\declaretheorem[numberwithin=section]{assumption}
\declaretheorem[]{proof sketch}

\usepackage{amsfonts}[mathscr]
\usepackage{amssymb}
\usepackage{amsmath}
\DeclareMathOperator*{\EV}{\mathbb{E}}

\newcommand{\A}{\mathcal{A}}
\newcommand{\Lfunc}{\mathcal{L}}

\newcommand{\X}{\mathcal{X}}

\newcommand{\F}{\mathcal{F}}

\newcommand{\D}{\mathcal{D}}

\newcommand{\Q}{\mathcal{Q}}
\newcommand{\G}{\mathcal{G}}
\newcommand{\mypar}[1]{\textbf{#1.}}

\newcommand{\entropy}{\mathcal{H}}
\newcommand{\Operator}{\mathcal{O}}
\newcommand{\der}{\mathrm{d}}

\newcommand{\RewardGuidedFineTuningSolver}{\textsc{\small{RewardGuidedFineTuningSolver}}\xspace}
\newcommand{\RewardGuidedFineTuningSolverRunningCosts}{\textsc{\small{RewardGuidedFineTuningSolverRunningCosts}}\xspace}
\newcommand{\noise}{U}
\newcommand{\bias}{b}
\newcommand{\hist}{\mathcal{T}}

\newcommand{\step}{\gamma}

\usepackage{pifont}
\usepackage{makecell}
\definecolor{darkgreen}{RGB}{0,100,0}
\definecolor{darkorange}{RGB}{255,140,0}

\newcommand{\AlgNameLong}{Reward-Guided Flow Merging\xspace}
\newcommand{\AlgNameShort}{\textsc{\small{RFM}}\xspace}
\newcommand{\AlgNameDef}{\textbf{R}eward-Guided \textbf{F}low \textbf{M}erging  (\textsc{\small{RFM}}\xspace)}

\definecolor{myviolet}{rgb}{0.6, 0.4, 0.8}
\definecolor{mygreen}{rgb}{0.0, 0.5, 0.0}

\newcommand{\newmacro}[2]{\newcommand{#1}{{#2}}}		

\newcommand{\dual}{h}

\newcommand{\ctime}{t}

\newmacro{\temp}{\eta}		
\newmacro{\points}{\mathcal{Z}}		
\newmacro{\intpoints}{\points^{\circ}}		
\newmacro{\point}{\dual}		
\newmacro{\pointalt}{\alt\point}		

\newmacro{\ctimealt}{s}		
\newmacro{\cstart}{0}		

\newmacro{\horizon}{T}		

\newmacro{\vecfield}{V}		

\newmacro{\signal}{V}		
\newmacro{\error}{W}		
\newmacro{\brown}{W}		

\newmacro{\dstate}{Y}		

\newmacro{\flowmap}{\Theta}		
\newcommandtwoopt{\flow}[2][\ctime][\point]{\flowmap_{#1}(#2)}
\newmacro{\minmax}{\Phi}		

\newmacro{\minvar}{x}		
\newmacro{\minvaralt}{\alt x}		
\newmacro{\minvars}{\mathcal{X}}		

\newmacro{\maxvar}{y}		
\newmacro{\maxvaralt}{\alt y}		
\newmacro{\maxvars}{\mathcal{Y}}		

\newmacro{\minsol}{\sol[\minvar]}		
\newmacro{\maxsol}{\sol[\maxvar]}		

\newcommand{\sol}[1][\point]{#1^{\ast}}		

\newmacro{\set}{\mathcal{S}}		

\newmacro{\open}{\mathcal{U}}		
\newmacro{\closed}{\mathcal{C}}		
\newmacro{\cpt}{\mathcal{K}}		
\newmacro{\nhd}{\mathcal{U}}		

\newacro{APT}{asymptotic pseudotrajectory}
\newacroplural{APT}[APTs]{asymptotic pseudotrajectories}
\newacro{GD}{gradient dynamics}
\newacro{GF}{gradient flow}
\newacro{ICT}{internally chain-transitive}
\newacro{MDS}{martingale difference sequence}
\newacro{NE}{Nash equilibrium}
\newacroplural{NE}[NE]{Nash equilibria}
\newacro{ODE}{ordinary differential equation}
\newacro{SA}{stochastic approximation}
\newacro{SFO}{stochastic first-order oracle}
\newacro{SG}{stochastic gradient}
\newacro{SP}{saddle-point}
\newacro{WAC}{weak asymptotic coercivity}

\newacro{AH}{Arrow\textendash Hurwicz}
\newacro{BDG}{Burkholder\textendash Davis\textendash Gundy}
\newacro{ConO}{consensus optimization}
\newacro{RM}{Robbins\textendash Monro}
\newacro{KW}{Kiefer\textendash Wolfowitz}
\newacro{GDA}{gradient descent/ascent}
\newacro{SGA}{symplectic gradient adjustment}
\newacro{SGD}{stochastic gradient descent}
\newacro{SGDA}{stochastic gradient descent/ascent}
\newacro{SPSA}{simultaneous perturbation stochastic approximation}
\newacro{ASGDA}[alt-SGDA]{alternating stochastic gradient descent/ascent}
\newacro{SEG}{stochastic extra-gradient}
\newacro{EG}{extra-gradient}
\newacro{PEG}{Popov's extra-gradient}
\newacro{RG}{reflected gradient}
\newacro{OG}{optimistic gradient}
\newacro{PPM}{proximal point method}

\newacro{GAN}{generative adversarial network}
\newacro{NN}{neural network}
\newacro{FTRL}{``follow the regularized leader''}
\newacro{CGD}{Competitive Gradient Descent}
\newacro{wp1}[w.p.$1$]{with probability $1$}

\definecolor{pastelblueold}{RGB}{56,146,236}
\definecolor{pastelblue}{RGB}{43,115,187}
\definecolor{pastelgreen}{RGB}{63,159,95}

\hypersetup{
  colorlinks=true,
  citecolor=pastelgreen,
  linkcolor=pastelblue,    
  urlcolor=pastelblue      
}

\title{A Unified Density Operator View of \\  Flow Control and Merging}

\author{
Riccardo {De Santi}\textsuperscript{$1,2$}, Malte Franke\textsuperscript{$1$}, Ya-Ping Hsieh\textsuperscript{$1$}, Andreas Krause\textsuperscript{$1,2$} \\
\textsuperscript{$1$}ETH Zürich, \textsuperscript{$2$}ETH AI Center \\
\texttt{\{rdesanti,krausea\}@ethz.ch}, \texttt{\{malte.franke,yaping.hsieh\}@inf.ethz.ch}
}

\begin{document}

\maketitle

\begin{abstract}
\input{sections/abstract}
\end{abstract}

\addtocontents{toc}{\protect\setcounter{tocdepth}{-1}}

\section{Introduction} 
\label{sec:introduction}
\input{sections/introduction}

\section{Background and Notation} 
\label{sec:background}
\input{sections/background}

\section{Reward-Guided Flow Merging via Implicit Density Operators} 
\label{sec:problem_setting}
\input{sections/problem_setting}

\section{Algorithm: \AlgNameLong} 
\label{sec:algorithm}
\input{sections/algorithm}

\section{Scalable Intersection via Flow Process Optimization} 
\label{sec:running_cost_KL}
\input{sections/running_cost_KL}


\section{Guarantees for \AlgNameLong} 
\label{sec:theory}
\input{sections/theory}

\section{Experimental Evaluation} 
\label{sec:experiments}
\input{sections/experiments}

\section{Related Work} 
\label{sec:related_works}
\input{sections/related_works}

\section{Conclusion} 
\label{sec:conclusions}
\input{sections/conclusions}

\section*{Acknowledgements}
This publication was made possible by the ETH AI Center doctoral fellowship to Riccardo De Santi. The project has received funding from the Swiss
National Science Foundation under NCCR Catalysis grant number 180544 and NCCR Automation grant agreement 51NF40 180545. 

\bibliography{biblio}
\bibliographystyle{iclr2026_conference}

\newpage
\appendix
\section{Appendix}
\tableofcontents
\newpage
\addtocontents{toc}{\protect\setcounter{tocdepth}{2}}

\section{Proofs for Section \ref{sec:theory}}
\label{sec:app_theory2}
\input{appendices/app_theory2}

\newpage

\section{Derivations of Gradients of First Variation}
\input{appendices/gradients_derivations}
\label{sec:gradients_derivations}
\newpage

\section{Proof of Proposition \ref{proposition:union_operator_mixture}}
\input{appendices/proof_proposition_union}
\label{sec:proof_proposition_union}
\newpage

\section{\AlgNameDef $ $ Implementation}
\input{appendices/alg_implementation}

\label{sec:app_alg_implementation}
\newpage

\section{\AlgNameDef: $ $ Computational Complexity, Cost, and Approximate Fine-Tuning Oracles}
\input{appendices/app_computational_cost}
\label{sec:app_computational_cost}
\newpage

\section{Experimental Details}
\input{appendices/experimental_details}
\label{sec:experimental_details}
\newpage

\end{document}

%% file: math_commands.tex
\usepackage{amsmath,amsfonts,bm}

\def\eqref#1{equation~\ref{#1}}

\def\1{\bm{1}}

\def\mF{{\bm{F}}}

\def\mP{{\bm{P}}}

\DeclareMathAlphabet{\mathsfit}{\encodingdefault}{\sfdefault}{m}{sl}
\SetMathAlphabet{\mathsfit}{bold}{\encodingdefault}{\sfdefault}{bx}{n}

\newcommand{\R}{\mathbb{R}}

\newcommand{\KL}{D_{\mathrm{KL}}}

\DeclareMathOperator*{\argmax}{arg\,max}
\DeclareMathOperator*{\argmin}{arg\,min}

%% file: sections/abstract.tex
\looseness -1 Recent progress in large-scale flow and diffusion models raised two fundamental algorithmic challenges: $(i)$ control-based reward adaptation of pre-trained flows, and $(ii)$ integration of multiple models, i.e., flow merging. While current approaches address them separately, we introduce a unifying probability-space framework that subsumes both as limit cases, and enables \emph{reward-guided flow merging}, allowing principled, task-aware combination of multiple pre-trained flows (e.g., merging priors while maximizing drug-discovery utilities). Our formulation renders possible to express a rich family of \emph{operators over generative models densities}, including intersection (e.g., to enforce safety), union (e.g., to compose diverse models), interpolation (e.g., for discovery), their reward-guided counterparts, as well as complex logical expressions via \emph{generative circuits}. Next, we introduce \AlgNameDef, a mirror-descent scheme that reduces reward-guided flow merging to a sequence of standard fine-tuning problems. Then, we provide first-of-their-kind theoretical guarantees for reward-guided and \emph{pure} flow merging via \AlgNameShort. Ultimately, we showcase the capabilities of the proposed method on illustrative settings providing visually interpretable insights, and apply our method to high-dimensional de-novo molecular design and low-energy conformer generation.

%% file: sections/introduction.tex
\looseness -1 Large-scale generative modeling has recently progressed at an unprecedented pace, with flow~\citep{lipman2022flow, lipman2024flow} and diffusion models~\citep{sohl2015deep, song2019generative, ho2020denoising} delivering high-fidelity samples in chemistry~\citep{hoogeboom2022equivariant}, biology~\citep{corso2022diffdock}, and robotics~\citep{chi2024diffusionpolicyvisuomotorpolicy}. However, adoption in real-world applications like scientific discovery led to two fundamental algorithmic challenges: $(i)$ reward-guided fine-tuning, \ie adapting pre-trained flows to maximize downstream utilities (\eg binding affinity)~\citep[\eg][]{domingo2024adjoint, uehara2024feedback, santi2025flow}, and $(ii)$ flow merging (FM): integrating multiple pre-trained models into one~\citep{song2023consistency, ma2025decouple}, \eg to incorporate safety constraints~\citep{dai2023safe}, or unify diverse priors~\citep{ma2025decouple}. 
Crucially, so far these two problems have been treated via fundamentally distinct formulations and methods. On the contrary, in this work we ask:
\begin{center}
    \emph{Can control and merging be cast in a single framework\\ enabling task-aware merging of multiple flows?}
\end{center} 
\looseness -1 Intuitively, this would allow to merge flows in a task-aware manner, as well as adapt flows to optimize downstream rewards while leveraging information from multiple prior models. Answering this would contribute to the algorithmic-theoretical foundations of \emph{flow adaptation}.
\paragraph{Our approach}
\looseness -1 To address this challenge, we first introduce a probability-space optimization framework (see Fig. \ref{fig:prob_opt_viewpoint}) that recovers reward-guided fine-tuning and \emph{pure} model merging as limit cases, and provably enables \emph{reward-guided flow merging} (Sec. \ref{sec:problem_setting}). Our formulation allows to express a rich family of operators over generative models, covering practical needs such as enforcing safety (\eg via intersection), composing diverse models (\eg via union), and discovery in data-scarce regions (\eg via interpolation). However, these operators are expressed via non-linear functionals that cannot be optimized via classic RL or control schemes, as shown by \citet{santi2025flow}. To overcome this challenge, we introduce \AlgNameDef, a mirror descent (MD)~\citep{nemirovskij1983problem} scheme that solves reward-guided and pure flow merging via a sequential adaptation process implementable via established fine-tuning methods~\citep[\eg][]{domingo2024adjoint, uehara2024feedback} (Sec. \ref{sec:algorithm}). Next, we extend the algorithm proposed, to operate on the space of entire flow processes, enabling scalable and stable computation of the intersection operator (Sec. \ref{sec:running_cost_KL}). We provide a rigorous convergence analysis of \AlgNameShort, yielding first-of-its-kind theoretical guarantees for reward-guided and pure flow merging (Sec. \ref{sec:theory}). Ultimately, we showcase our method's capabilities on illustrative settings, as well as on high-dimensional molecular design and conformer generation tasks (Sec. \ref{sec:experiments}).

\paragraph{{Our contributions}} To sum up, in this work we contribute

\begin{itemize}[noitemsep,topsep=0pt,parsep=0pt,partopsep=0pt,leftmargin=*]
    \item \looseness -1 A formalization of \emph{reward-guided flow merging} via density operators, which unifies and generalizes reward-guided adaptation and flow merging (Sec. \ref{sec:problem_setting}).
    \item \looseness -1 {\em \AlgNameDef}, a probability-space optimization algorithm that provably solves reward-guided flow merging, and pure flow merging problems (Sec. \ref{sec:algorithm}), and a stability-enhancing extension for flow intersection, implementing a scalable mirror-descent scheme over the space of flow processes (Sec. \ref{sec:running_cost_KL}).
    \item \looseness -1 Theoretical convergence guarantees for the proposed algorithms, leveraging recent insights on mirror flows, and a novel technical result of independent interest. (Sec. \ref{sec:theory}).
    \item \looseness -1 An experimental evaluation of \AlgNameShort showcasing its practical relevance on both synthetic, yet illustrative settings, as well as on scientific discovery tasks, namely molecular design and conformer generation. (Sec. \ref{sec:experiments}).
\end{itemize}

%% file: sections/background.tex
\looseness -1 We denote the set of Borel probability measures on a set $\X$ with $\mP(\X)$, and the set of functionals over $\mP(\X)$ as $\mF(\X)$. 

\mypar{Generative Flow Models}
Generative models aim to approximately sample novel data points from a data distribution $p_{data}$. Flow models tackle this problem by transforming samples $X_0 = x_0$ from a source distribution $p_0$ into samples $X_1=x_1$ from the target distribution $p_{data}$~\citep{lipman2024flow, farebrother2025temporal}. Formally, a \emph{flow} is a time-dependent map $\psi: [0,1]\times \R^d \to \R$ such that $\psi: (t,x) \to \psi_t(x)$. A \emph{generative flow model} is a continuous-time Markov process $\{X_t\}_{0 \leq t \leq 1}$ obtained by applying a flow $\psi_t$ to $X_0 \sim p_0$ as $X_t = \psi_t(X_0)$, $ t \in [0,1]$, such that $X_1 = \psi_1(X_0) \sim p_{data}$. In particular, the flow $\psi$ can be defined by a \emph{velocity field} $u: [0,1] \times \R^d \to \R^d$, which is a vector field related to $\psi$ via the following ordinary differential equation (ODE), typically referred to as \emph{flow ODE}:
\begin{equation}
    \frac{\der}{\der t} \psi_t(x) = u_t(\psi_t(x)) \label{eq:flow_diff_eq} 
\end{equation} 
with initial condition $\psi_0(x) = 0$. 
A flow model $X_t = \psi_t(X_0)$ induces a probability path of \emph{marginal densities} $p = \{p_t\}_{0 \leq t \leq 1}$ such that at time $t$ we have that $X_t \sim p_t$. 
We denote by $p^u$ the probability path of marginal densities induced by the velocity field $u$. Flow matching (FM)~\citep{lipman2024flow} can estimate a velocity field $u^{\theta}$ s.t. the induced marginal densities $p^{u_\theta}$ satisfy $p^{u_\theta}_0 = p_0$ and $p^{u_\theta}_1 = p_{data}$, where $p_0$ denotes the source distribution, and $p_{data}$ the target data distribution. Typically FM are rendered tractable by defining $p^u_t$ as the marginal of a conditional density $p^u_t(\cdot | x_0, x_1)$, \eg:
\begin{equation}
    X_t \; |\; X_0, X_1 = \kappa_t X_0 + \omega_t X_1
\end{equation}
\looseness -1 where $\kappa_0 = \omega_1 = 1$ and $\kappa_1 = \omega_0 = 0$ (e.g. $\kappa_t = 1- t$ and $\omega_t = t$). Then $u^\theta$ can be learned by regressing onto the conditional velocity field $u(\cdot | x_1)$ \citep{lipman2022flow}. 
As diffusion models~\citep{song2019generative} (DMs) admit an equivalent ODE formulation~\citep[Ch. 10]{lipman2024flow}, our contributions extend directly to DMs.

\mypar{Continuous-time Reinforcement Learning}
\looseness -1 We formulate finite-horizon continuous-time RL as a specific class of optimal control problems \citep{wang2020reinforcement, zhao2024scores}. Given a state space $\X$ and an action space $\A$, we consider the transition dynamics governed by the following ODE:
\begin{equation} 
      \frac{\der}{\der t} \psi_t(x) = a_t(\psi_t(x)) \label{eqn_continuous_RL} 
\end{equation}
\looseness -1 where $a_t \in \A$ is an action. We consider a state space $\X \coloneqq \R^d \times [0,1]$, and denote by (Markovian) deterministic policy a function $\pi_t(X_t) \coloneqq \pi(X_t, t) \in \A$ mapping a state $(x,t) \in \X$ to an action $a \in \A$ such that $a_t = \pi(X_t, t)$, and denote with $p_t^\pi$ the marginal density at time $t$ induced by policy $\pi$. 

\mypar{Pre-trained Flow Models as an RL policy}
\looseness -1 A pre-trained flow model with velocity field $u^{pre}$ can be interpreted as an action process $a^{pre}_t \coloneqq u^{pre}(X_t, t)$, where $a_t^{pre}$ is determined by a continuous-time RL policy via $a_t^{pre} = \pi^{pre}(X_t, t)$~\citep{de2025provable}. Therefore, we can express the flow ODE induced by a pre-trained flow model by replacing $a_t$ with $a^{pre}$ in Eq. \eqref{eqn_continuous_RL}, and denote the pre-trained model by its policy $\pi^{pre}$, which induces a density $p^{pre}_1 \coloneqq p^{\pi^{pre}}_1$ approximating $p_{data}$.

%% file: sections/problem_setting.tex
\begin{figure*}[t]
    \centering
    \begin{subfigure}{0.59\textwidth} 
        \centering
        \includegraphics[width=\textwidth, keepaspectratio]{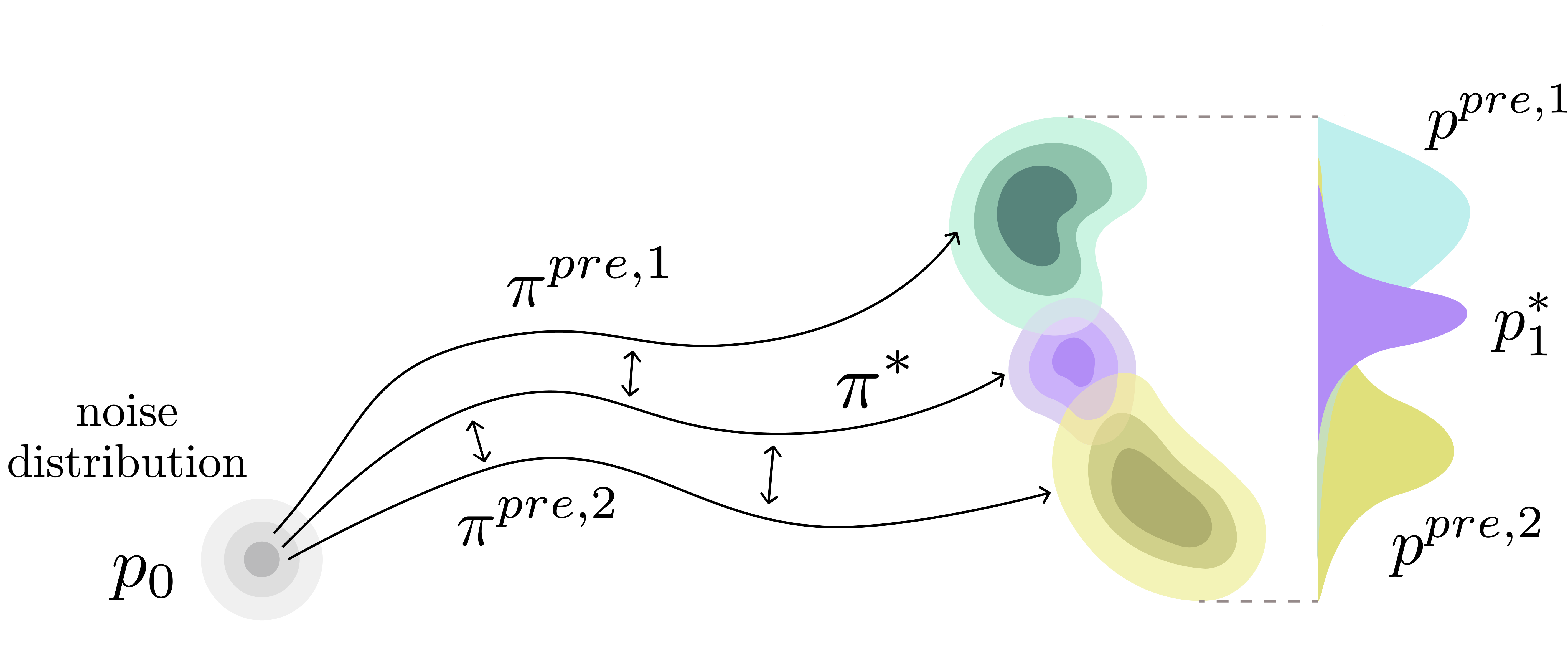}
        \caption{Reward-Guided Flow Merging}
        \label{fig:process_drawing}
    \end{subfigure}%
    \hspace{10pt}
    \begin{subfigure}{0.37\textwidth} 
        \centering
        \raisebox{-1ex}[0pt][0pt]{
        \includegraphics[width=0.62\textwidth, keepaspectratio]{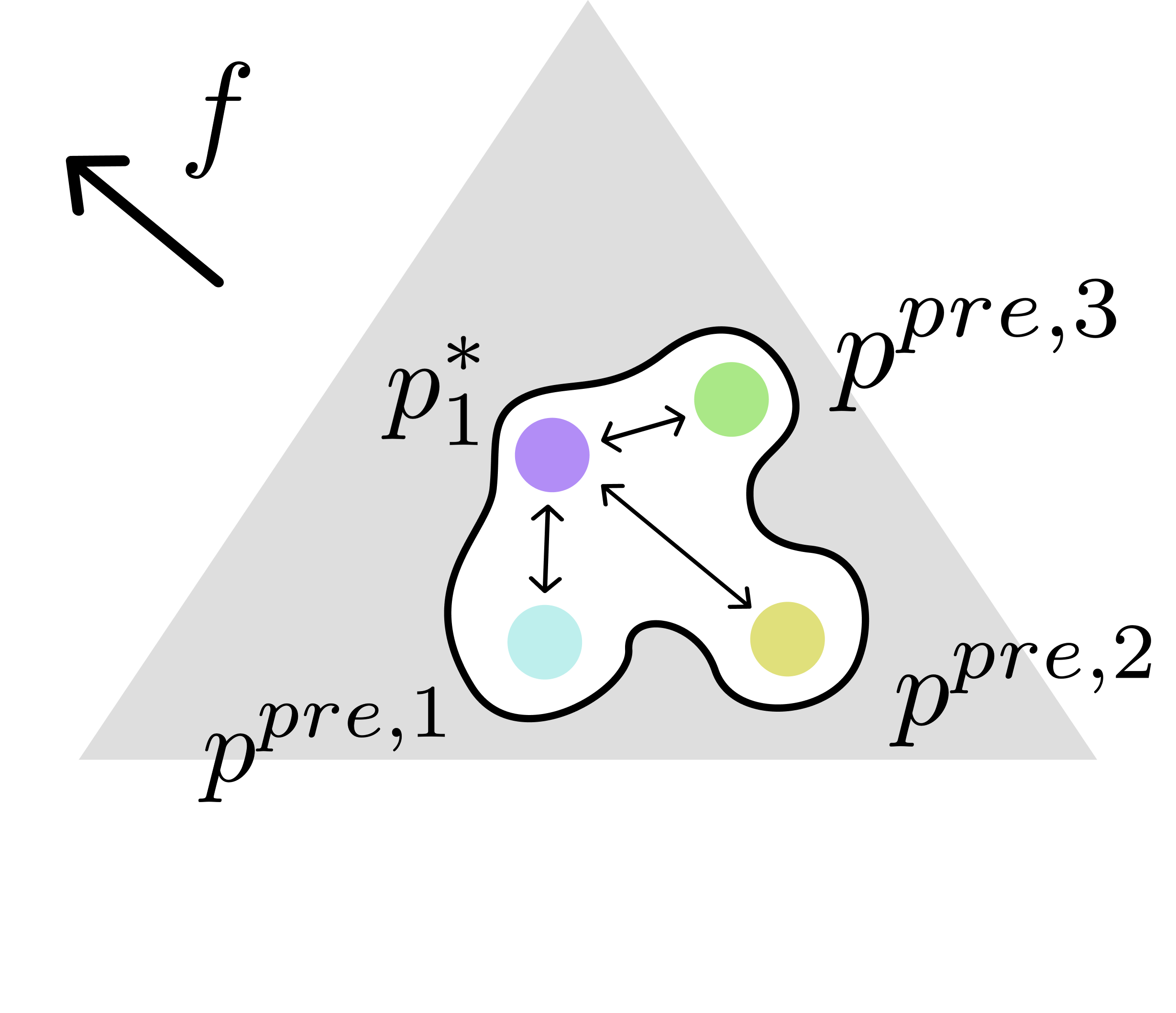}
        }
        \caption{Probability-Space Opt. Viewpoint}
        \label{fig:prob_opt_viewpoint}
    \end{subfigure}
    \caption{\looseness -1 (\ref{fig:process_drawing}) Pre-trained and fine-tuned policies inducing $\{p_1^{pre,i}\}_{i=1}^n$ and optimal density $p_1^*$ computed via flow merging, i.e., subcase of Problem \ref{eq:reward_guided_flow_merging_problem} where $f$ is disregarded.  (\ref{fig:prob_opt_viewpoint}) Probability-space optimization viewpoint on reward-guided flow merging, as in Problem \ref{eq:reward_guided_flow_merging_problem}.}
    \label{fig:joint_top_figures} 
\end{figure*}
\looseness -1 In this section, we introduce the general problem of \emph{reward-guided flow merging} via \emph{density operators}. Formally, we wish to implement an operator $\mathcal{O}$: $\Pi \times \ldots \times \Pi \to \Pi$ that, given pre-trained generative flow models $\{\pi^{pre,i}\}_{i\in[n]}$, returns a merged flow $\pi^*$ inducing an ODE:
\begin{equation} 
      \frac{\der}{\der t} \psi_t(x) = a^*_t(\psi_t(x)) \quad  \text{with} \quad a_t^* = \pi^*(x,t), \label{eq:controlled_ODE}
\end{equation}
\looseness -1 such that it controllably merges prior information within the $n$ pre-trained generative models, while potentially steering its density $p_1^* \coloneqq p_1^{\pi^*}$ towards a high-reward region according to a given scalar reward function $f(x): \X \to \R$. We implement such operators by fine-tuning an initial flow $\pi^{init} \in \{\pi^{pre,i}\}_{i\in[n]}$ according to the following probability-space optimization problem (see Fig. \ref{fig:prob_opt_viewpoint}).
\begin{tcolorbox}[colframe=white!, top=2pt,left=2pt,right=2pt,bottom=2pt]
\begin{center}
\textbf{Reward-Guided FM via Density Operator}
\begin{align}
    \pi^* \in \argmax_{\pi : p_0^* = p_0^{pre}}\;  \EV_{x\sim p_1^\pi} \left[ f(x) \right] - \sum_{i=1}^n \alpha_i \D_i( p_1^{\pi}  \, \| \,  p_1^{pre,i})  \label{eq:reward_guided_flow_merging_problem}
\end{align}
\end{center}
\end{tcolorbox}
\looseness -1  Here, each $D_i$ is an arbitrary divergence, $\alpha_i > 0$ are model-specific weights, and $p_0^\pi = p_0^{pre}$ enforces that the marginal density at $t=0$ must match the pre-trained model marginal. 
This formulation recovers reward-guided fine-tuning~\citep[\eg][]{domingo2024adjoint} when $n=1$ and $\D_1 =  D_{KL}$, and provides a formal framework for \emph{pure} flow merging~\citep[\eg][]{poole2022dreamfusion, song2023consistency} with interpretable objectives, when the reward $f$ is constant (\eg $f(x) = 0 \; \forall x \in \X$).
\looseness -1  In this case, Eq. \ref{eq:reward_guided_flow_merging_problem} formalizes flow merging as computing a flow $\pi^*$ that minimizes a weighted sum of divergences to the priors $\{\pi^{pre,i}\}_{i\in[n]}$. Varying the divergences $\{D_i\}_{i\in [n]}$ yields different merging strategies.
\paragraph{In-Distribution Flow Merging.}
\looseness -1 Given pre-trained flow models $\{\pi^{pre,i}\}_{i\in[n]}$, we denote by \emph{in-distribution} merging when the merged model generates samples from regions with sufficient prior density. Practically relevant instances include the \emph{intersection operator} $\mathcal{O}_{\land}$ (\ie a logical AND), and the \emph{union operator} $\mathcal{O}_{\lor}$ (\ie a logical OR). Formally, these operators can be defined as follows:

\begin{minipage}{0.49\textwidth}
\begin{tcolorbox}[colframe=white, colback=gray!10, top=2pt,left=2pt,right=2pt,bottom=2pt]
\begin{center}
\textbf{$\mathcal{O}_{\land}$: Intersection ($\land$) Operator}
\begin{align}
    \pi^* \in \argmin_{\pi : p_0^* = p_0^{pre}}\; \sum_{i=1}^n \alpha_i \, D_{KL}(p_1^\pi\|p_1^{pre,i}) \label{eq:and_operator_problem}
\end{align}
\end{center}
\end{tcolorbox}
\end{minipage}\hfill
\begin{minipage}{0.49\textwidth}
\begin{tcolorbox}[colframe=white, colback=gray!10, top=2pt,left=2pt,right=2pt,bottom=2pt]
\begin{center}
\textbf{$\mathcal{O}_{\lor}$: Union ($\lor$) Operator}
\begin{align}
    \pi^* \in \argmin_{\pi : p_0^* = p_0^{pre}}\; \sum_{i=1}^n \alpha_i \, D^R_{KL}(p_1^\pi \| p_1^{pre,i}) \label{eq:or_operator_problem}
\end{align}
\end{center}
\end{tcolorbox}
\end{minipage}

\looseness -1 The $D_{KL}$ divergences in Eq. \ref{eq:and_operator_problem} heavily penalize density allocation in any region with low prior density for any model $\pi^{pre,i}$, leading to an optimal flow model $\pi^*$ inducing  \smash{$p_1^*(x) \propto \prod_{i=1}^n p_1^{pre,i}(x)^{\alpha_i}$}~\citep[cf.][]{heskes1997selecting}. Similarly, the reverse KL divergence \smash{$D^R_{KL}(p \| q) \coloneqq D_{KL}(q \| p)$} in Eq. \ref{eq:or_operator_problem} induces a mode-covering behaviour implying a flow model $\pi^*$ with density \smash{$p_1^* \propto \sum_{i=1}^n \alpha_i p_1^{pre,i}(x)$}~\citep[cf.][]{banerjee2005clustering} sufficiently covering all regions with enough prior density, for any $p_1^{pre,i}, \; i \in [n]$.
\paragraph{Out-of-Distribution Flow Merging.} \looseness -1 We denote by \emph{out-of-distribution}, the case where $\pi^*$ samples from regions insufficiently covered by all priors. An example is the \emph{interpolation operator} $\mathcal{O}_{W_p}$ (see Eq. \ref{eq:interpolation_operator_problem}), inducing $p_1^*$ equal to the priors Wasserstein Barycenter~\citep{cuturi2014fast}.
\begin{tcolorbox}[colframe=white!, top=2pt,left=2pt,right=2pt,bottom=2pt]
\begin{center}
\textbf{$\mathcal{O}_{W_p}$: Interpolation (Wasserstein-$p$ Barycenter) Operator}
\begin{align}
    \argmin_{\pi}\;  \sum_{i=1}^n \alpha_i W_p( p_1^{\pi}  \, \| \,  p_1^{pre,i}) \label{eq:interpolation_operator_problem}
\end{align}
\end{center}
\end{tcolorbox}
\paragraph{Straightforward Generalizations.}
\looseness -1 While we presented a few practically relevant operators, the framework in Eqs. \ref{eq:reward_guided_flow_merging_problem} is not tied to them: it trivially admits any new operator defined via other divergences (\eg MMD, Rényi, Jensen–Shannon), and allows diverse $D_i$ for each prior flow models $\pi^{pre,i}$. Moreover, sequential composition of these operators makes it possible to implement arbitrarily complex logical operations over generative models. For instance, as later shown in Sec. \ref{sec:experiments}, one can obtain $\pi^* = (\pi^{pre,1} \lor \pi^{pre,2}) \land \pi^{pre,3}$ by first computing $\pi_{1,2} \coloneqq \Operator_{\lor}(\pi^{pre,1}, \pi^{pre,2})$ and then $\pi^* \coloneqq \Operator_{\land}(\pi_{1,2}, \pi^{pre,3})$. We denote such operators by \emph{generative circuits}, and illustrate one in Fig. \ref{fig:toy2_top_d}.

\looseness -1 While being of high practical relevance, the presented framework entails optimizing non-linear distributional utilities (see Eq. \ref{eq:reward_guided_flow_merging_problem}) beyond the reach of standard RL or control schemes, as shown by~\citet{santi2025flow}. In the next section, we show how to reduce the introduced problem to sequential fine-tuning for maximization of rewards automatically determined by the choice of operator $\mathcal{O}$.

%% file: sections/algorithm.tex
\looseness -1 In this section, we introduce \AlgNameDef, see Alg. \ref{alg:algorithm}, which provably solves Problem \ref{eq:reward_guided_flow_merging_problem}. \AlgNameShort implements general operators $\Operator$ (see Sec. \ref{sec:problem_setting}) by solving the following problem:
\begin{tcolorbox}[colframe=white!, top=2pt,left=2pt,right=2pt,bottom=2pt]
\begin{center}
\textbf{Reward-Guided Flow Merging as Probability-Space Optimization}
    \begin{align}
        p_1^{\pi^*} \in \argmax_{p_1^\pi}\;  \G(p_1^\pi) \quad \text{ with } \quad \G(p_1^\pi) \coloneqq \EV_{x\sim p_1^\pi} \left[ f(x) \right] - \sum_{i=1}^n \alpha_i \D_i( p_1^{\pi}  \, \| \,  p_1^{pre,i}) \label{eq:probability_space_viewpoint}
    \end{align}
\end{center}
\end{tcolorbox}
Given an initial flow model \smash{$\pi^{init} \in \{\pi^{pre,i}\}_{i\in[n]}$}, \AlgNameShort follows a mirror descent (MD) scheme~\citep{nemirovskij1983problem} for $K$ iterations by sequentially fine-tuning $\pi^{init}$ to maximize surrogate rewards $g_k$ determined by the chosen operator, \ie $\G$. To understand how \AlgNameShort computes the surrogate rewards $\{g_k\}_{k=1}^K$ guiding the optimization process in Eq. \ref{eq:probability_space_viewpoint}, we first recall the notion of first variation of $\G$ over a space of probability measures~\citep[cf.][]{hsieh2019finding}. A functional $\G \in \mF(\X)$ has a first variation at $\mu \in \mP(\X)$ if there exists a function $\delta \G(\mu) \in \mF(\X)$ such that:
\begin{equation*}
    \G(\mu + \epsilon \mu') = \G(\mu) + \epsilon \langle \mu', \delta \G(\mu) \rangle + o(\epsilon). 
\end{equation*}
\looseness -1 holds for all $\mu' \in \mP(\X)$, where the inner product is an expectation. At iteration $k \in [K]$, given the current generative model $\pi^{k-1}$, \AlgNameShort fine-tunes it according to the following standard entropy-regularized control or RL problem, solvable via any established method~\citep[\eg][]{domingo2024adjoint}
\begin{equation}
    \argmax_{\pi}\quad  \langle \delta \G \left( p^{\pi_{k-1}}_1 \right), p_1^{\pi} \rangle - \frac{1}{\gamma_k} D_{KL}(p_1^{\pi} \,\|\, p_1^{\pi_{k-1}})\label{eq:opt_first_variation}
\end{equation}
Thus, we introduce a surrogate reward function $g_k: \X \to \R$ defined for all $x \in \X$ such that:
\begin{equation}
    g_k(x) \coloneqq \delta \G \left(p^{\pi^{k-1}}_1 \right) (x) \quad \text{ and } \quad \EV_{x \sim p_1^\pi}[g_k(x)] = \langle \delta \G \left( p^{\pi^{k-1}}_1 \right), p_1^{\pi} \rangle  \label{eq:g_from_functional_linearized}
\end{equation}
\begin{algorithm}[t] 
    \small
    \caption{\AlgNameDef} 
    \label{alg:algorithm}
        \begin{algorithmic}[1]
        \State{\textbf{input: } \looseness -1 $\{\pi^{pre,i}\}_{i\in [n]}: $ pre-trained flows, $\{\D_i\}_{i\in [n]}: $ arbitrary divergences, $f:$ reward, $\{\alpha_i\}_{i\in [n]}:$ weighs, $K: $ iterations number, $\{\gamma_k \}_{k=1}^{K}$ stepsizes, $\pi^{init} \in \{\pi^{pre,i}\}_{i\in [n]}:$ initial flow model}
        \State{\textbf{Init:} $\pi_0 \coloneqq \pi^{init} $}
        \For{$k=1, 2, \hdots, K$}
        \State{Estimate $\nabla_x g_{k} = \nabla_x \delta \G(p_1^{\pi^{k-1}})$ with:
        \begin{equation}
        \begin{aligned}
        \G \left(p_1^{\pi^{k-1}}\right) &=
        \left\{
        \begin{aligned}
        &\EV_{x\sim p_1^{\pi^{k-1}}} \left[ f(x) \right] - \sum_{i=1}^n \alpha_i \D_i( p_1^{\pi^{k-1}}  \, \| \,  p_1^{pre,i})
        && \text{(Reward-Guided Flow Merging)} \\
        & - \sum_{i=1}^n \alpha_i \D_i( p_1^{\pi^{k-1}}  \, \| \,  p_1^{pre,i})
        && \text{(Flow Merging)}
        \end{aligned}
        \right.
        \end{aligned}
        \end{equation}
        }
        \State{Compute $\pi_k$ via standard reward-guided fine-tuning~\citep[\eg][]{domingo2024adjoint}:
            \begin{equation*}
                \pi_k \leftarrow \RewardGuidedFineTuningSolver(\nabla_x g_k, \gamma_k, \pi_{k-1}) 
            \end{equation*}
        }
        \EndFor
        \State{\textbf{output: } policy $\pi \coloneqq \pi_{K}$} 
        \end{algorithmic}
\end{algorithm}
\looseness -1 We now present \AlgNameDef, see Alg. \ref{alg:algorithm}. 
At each iteration $k \in [K]$, \AlgNameShort estimates the gradient of the first variation at the previous policy $\pi_{k-1}$, \ie \smash{$\nabla_x \delta \G ( p^{\pi^{k-1}}_1)$} (line $4$). Then, it updates the flow model $\pi_k$ by solving the reward-guided fine-tuning problem in Eq. \ref{eq:opt_first_variation} by employing \smash{$\nabla_x g_k \coloneqq \nabla_x \delta \G ( p^{\pi^{k-1}}_1 )$} as reward function gradient (line $5$). Ultimately, \AlgNameShort returns a final policy $\pi \coloneqq \pi_K$.
We report a detailed implementation of \RewardGuidedFineTuningSolver in Apx. \ref{sec:app_alg_implementation}.

\paragraph{Implementation of Intersection, Union, and Interpolation operators.}
In the following, we present the specific expressions of $\nabla_x \delta \G(p_1^{\pi})$ for pure model merging with the intersection ($\Operator_\land$), union ($\Operator_\lor$), and interpolation ($\Operator_{W_p}$) operators introduced in Sec. \ref{sec:problem_setting}. 
\begin{equation*}
    \nabla_x \delta \G(p_1^{\pi})(x) = 
    \begin{cases}
        - \sum_{i=1}^n \alpha_i s^{k-1}(x, t=1) + \sum_{i=1}^n \alpha_i s^{\pi^{pre,i}}(x,t=1) & \text{Intersection ($\Operator_\land$)}\\
        - \sum_{i=1}^n \nabla_x \exp{(\phi_i^*(x)-1)}, \phi_i^* \text{ as by Eq. \ref{eq:first_var_reverse_KL}} & \text{Union ($\Operator_\lor$)} \\
        - \sum_{i=1}^n \nabla_x \phi_i^*(x), \phi_i^* = \argmax_{\phi: \|\nabla_x \phi\|\leq 1} \langle \phi, p^\pi - p^{pre,i}\rangle & \text{Interpol. ($\Operator_{W_1}$)}
    \end{cases}
\end{equation*}
\looseness -1 Where by $s^{k-1}(x,t) \coloneqq \nabla \log p^{\pi-1}_t(x)$ we denote the score of model $\pi^{k-1}$ at point $x$ and time $t$, and \smash{$s^{pre,i} \coloneqq s^{\pi^{pre,i}}$}. For diffusion models, a learned neural score network is typically available; for flows, the score follows from a linear transformation of $\pi(X_t,t)$~\citep[\eg][Eq. 8]{domingo2024adjoint}:
\begin{equation}
    s^\pi_t(x) = \frac{1}{\kappa_t(\tfrac{\dot\omega_t}{\omega_t}\kappa_t - \dot\kappa_t)}\left(\pi(x,t)  - \frac{\dot\omega_t}{\omega_t}x\right) \label{eq:score_via_vector_field} 
\end{equation}
\looseness -1 For the union operator, gradients are defined via critics $\{\phi^*_i\}_{i=1}^n$ learned with the standard variational form of reverse KL, as in f-GAN training of neural samplers~\citep{nowozin2016f}. For $W_1$ interpolation, each $\phi^*_i$ plays the role of a Wasserstein-GAN discriminator with established learning procedures~\citep{arjovsky2017wassersteingan}. In both cases, each critic compares the fine-tuned density to a prior density \smash{$p_1^{pre,i}$}, seemingly requiring one critic per prior. We prove that, surprisingly, this is unnecessary for the union operator, and conjecture that analogous results hold for other divergences.
\begin{restatable}[Union operator via Pre-trained Mixture Density Representation]{proposition}{UnionMixture}
\label{proposition:union_operator_mixture}
Given $\overline{p}_1^{pre} = \nicefrac{\sum_{i=1}^n \alpha_i p_1^{pre,i}}{\sum_{i=1}^n \alpha_i}$, \ie the $\alpha$-weighted mixture density of pre-trained models, the following hold: 
\begin{equation}
    \pi^* \in \argmin_{\pi}\; \sum_{i=1}^n \alpha_i \, D^R_{KL}(p_1^\pi \,\|\, p_1^{pre,i})
    = \argmin_{\pi}\; \left(\sum_{i=1}^n \alpha_i \right)  D^R_{KL}  (p_1^\pi \,\|\, \overline{p}_1^{pre})\label{eq:union_operator_mixture_eq} 
\end{equation}
\end{restatable}
Prop. \ref{proposition:union_operator_mixture}, which is proved in Apx. \ref{sec:proof_proposition_union} implies that the union operator in Eq. \ref{eq:or_operator_problem} over $n$ prior models can be implemented by learning a single critic $\phi^*$, as shown in Sec. \ref{sec:experiments}.
In Apx. \ref{sec:gradients_derivations}, we report the gradient expressions above, and present a brief tutorial to derive the first variations for any new operator.

\looseness -1 Crucially, the score in Eq. \ref{eq:score_via_vector_field}  for the intersection gradient diverges at $t=1$ ($\kappa_1 =0$). While prior works attenuate the issue by evaluating the score at $1-\epsilon$~\citep{de2025provable}, this trick hardly scales well to high-dimensional settings. In the following, we propose a principled solution to this problem by leveraging weighted score estimates along the entire noised flow process, \ie $t \in [0,1]$.

%% file: sections/running_cost_KL.tex
\looseness -1 To tackle the aforementioned issue, we lift the problem in Eq. \ref{eq:and_operator_problem} from the probability space of the last time-step marginal $p_1^\pi$, where the score diverges, to the entire flow process:
\vspace{-1.3mm}
\begin{tcolorbox}[colframe=white!, top=2pt,left=2pt,right=2pt,bottom=2pt]
\begin{center}
\textbf{Intersection Operator via Flow Process Optimization}
\begin{align}
      \argmax_{\pi : p_0^\pi = p_0^{pre}}\;\!\Lfunc_{\land}\!\left(\mathbf{Q}^\pi\right)\!\coloneqq\!\int_0^1 \!\lambda_t \sum_{i=1}^n \alpha_i \, D_{KL}(p_t^\pi \,\|\, p_t^{pre,i}) \; \der t  \label{eq:joint_process_and_operator}
\end{align}
\end{center}
\end{tcolorbox}
\vspace{-1.5mm}
\looseness -1 Here, \smash{$\mathbf{Q}^\pi = \{p^\pi_t\}_{t\in [0,1]}$} denotes the entire joint flow process induced by policy $\pi$ over $\X^{[0,1]}$. Under general regularity assumptions, an optimal policy $\pi^*$ for Problem \ref{eq:joint_process_and_operator} is optimal also \wrt Eq. \ref{eq:and_operator_problem}. Interestingly, an optimal flow $\pi^*$ for Problem \ref{eq:joint_process_and_operator} can be computed via a MD scheme acting over the space of joint flow processes $\mathbf{Q}^\pi = \{p^\pi_t\}_{t\in [0,1]}$ determined by the following update rule:
\begin{tcolorbox}[colframe=white!, top=2pt,left=2pt,right=2pt,bottom=2pt]
\begin{center}
\textbf{Reward-Guided FM (Mirror Descent)  Step}
\begin{equation}
      \mathbf{Q}^{k} \in \argmax_{q : p_0 = p^{k-1}_0} \langle \delta \Lfunc_{\land}(\mathbf{Q}^{k-1}), \mathbf{Q}\rangle + \frac{1}{\step^{k}} D_{KL}\left(  \mathbf{Q}  \| \mathbf{Q}^{k-1} \right) \label{eq:MD_step}
\end{equation}
\end{center}
\end{tcolorbox}
\looseness -1 First, we state the following Lemma \ref{lemma:first_var_process}, which allows to express the first variation of $\Lfunc_{\land}$ \wrt the flow process $\mathbf{Q^\pi}$ as an integral of first variations \wrt marginal densities $p_t^\pi$.
\begin{restatable}[First Variation of Flow Process Functional]{lemma}{FirstVarProcess}
\label{lemma:first_var_process}
For objective $\Lfunc_{\land}$ in Eq. \ref{eq:joint_process_and_operator} it holds the following:
    \begin{equation}
    \langle \delta \Lfunc_{\land}(\mathbf{Q}^k), q \rangle\!=\!\int_0^1\!\lambda_t\!\EV_{\quad \mathbf{Q}}\!\left[\!\delta \sum_{i=1}^n\!\alpha_i \, D_{KL}(p_t^\pi \,\|\, p_t^{pre,i})\!\right]\!\der t \label{eq:running_cost_term}
    \end{equation}
\end{restatable}
This factorization of $\langle \delta \Lfunc_{\land}(\mathbf{Q}^k), q \rangle$ shows that a flow $\pi_{k+1}$ inducing an optimal process $\mathbf{Q}^k$ \wrt the update step in Eq. \ref{eq:MD_step} can be computed by solving a control-affine optimal control problem via the same \RewardGuidedFineTuningSolver oracle used in Alg. \ref{alg:algorithm}, by introducing the running cost term: 
\begin{equation}
    f_t(x) \coloneqq \ \delta \left(\sum_{i=1}^n \alpha_i \, D_{KL}(p_t^\pi \,\|\, p_t^{pre,i})\right)(x,t)
\end{equation}
\looseness -1 with $\in [0,1)$. This algorithmic idea, which allows to control the score scale at $t \to 1$ via $\lambda_t$, thus enhancing \AlgNameShort,  trivially extends to reward-guided merging, and is accompanied by a detailed pseudocode in Apx. \ref{sec:app_alg_implementation}.

%% file: sections/theory.tex
\newcommand{\LinearFineTuningSolver}{\texttt{LinearFineTuningSolver}}
\newcommand{\Qmd}{\mathbf{Q}_\sharp}

\newcommand{\RefPro}{\mathbf{R}}
\newcommand{\drm}{\mathrm{d}}

\looseness -1 In this section, we aim to provide rigorous convergence guarantees for \AlgNameShort by interpreting its iterations as \emph{mirror descent} on the space of measures. To this end, at each iteration it must hold that $s^{k-1}(x,t) := \nabla_x \log p_t^{\pi^{k-1}}(x)$, \ie the \RewardGuidedFineTuningSolver subroutine, employed by \AlgNameShort at every iteration (see line $5$, Alg. \ref{alg:algorithm}), \emph{retains the score information}. Our key contribution is to prove this result, which is not only essential for our convergence analysis, but also of independent interest: it provides a rigorous justification for a structural assumption that underlies several control-based fine-tuning analyses, where it is typically implicitly assumed~\citep{santi2025flow,de2025provable}.

\noindent\textbf{Score Retention via Control-based Fine-Tuning.}  
\looseness -1 Our main observation is that, under a standard assumption, control-based fine-tuning schemes retain score information. These include Adjoint Matching \citep{domingo2024adjoint}, which we utilize in our experiments (Sec. \ref{sec:experiments}). We consider the standard stochastic optimal control (SOC) framework, exposed in Apx. \ref{sec:app_theory2}, and show that the fine-tuned model via SOC \emph{necessarily encodes} score information. 

\begin{tcolorbox}[colframe=white!, top=2pt,left=2pt,right=2pt,bottom=2pt]
\begin{restatable}[SOC Retains Score Information]{theorem}{AMretainsScores}
\label{theorem:AMretainsScores}
\looseness -1 Suppose that the prior diffusion model forward process converges to a standard Gaussian noise\footnote{ This is a standard assumption in diffusion modeling \citep[\eg][]{ho2020denoising, song2021scorebased}.}. Then, the model returned by a SOC fine-tuning solver is such that:
\begin{equation}
    u^\star(x,t) \coloneqq \sigma(t)\,\nabla \log p^k_t(x) 
\end{equation}
\looseness -1 where $p^k_t$ denotes the marginal distribution of the prior forward process, initialized at $p^{\pi_k}_1$, and $u^\star(x,t)$ the applied optimal control (see Sec. \ref{sec:app_theory2}). In other words, \RewardGuidedFineTuningSolver exactly recovers the score.
\end{restatable}
\end{tcolorbox}

\looseness -1  \cref{theorem:AMretainsScores} enables us to reinterpret \cref{alg:algorithm} as generating \emph{approximate mirror iterates}, a framework that has proven effective for sampling and generative modeling~\citep{karimi2024sinkhorn,de2025provable,santi2025flow}.

\noindent\textbf{Robust Convergence under Inexact Updates.}
Thanks to \cref{theorem:AMretainsScores}, we can now develop a rigorous convergence theory for \cref{alg:algorithm} under the realistic condition that \RewardGuidedFineTuningSolver (see Sec. \ref{sec:algorithm}) is implemented \emph{approximately}.
\looseness -1 Let $\G$ be the objective in Eq. \ref{eq:probability_space_viewpoint}. Via $\pi^k$, the iterates generated by \cref{alg:algorithm} induce a sequence of stochastic processes, denoted by $\mathbf{Q}^k$, which satisfy $\mathbf{Q}^k = p_1^{\pi^k}$.
 Each iterate $\mathbf{Q}^k$ is understood as an approximation to the \emph{idealized} mirror descent step: 
\begin{equation}
    \mathbf{Q}^k_\sharp \in \argmax_{\mathbf{Q} : p_0 = p_0^{pre}} 
    \Bigl\{ \langle \delta \G(p^{\pi_k}_1), \mathbf{Q}\rangle 
    - \tfrac{1}{\gamma^k} D_{KL}\!\left(\mathbf{Q} \,\|\, \mathbf{Q}^{k-1}\right) \Bigr\}.
    \label{eq:MD_step_copy}
\end{equation}
which serves as the exact reference point for our analysis. To quantify the discrepancy between $\mathbf{Q}^k$ and $\Qmd^k$, let $\hist_k$ denote the history up to step $k$, and decompose the error as
\begin{align}
    \bias_k &\coloneqq \EV\!\left[\delta \G(p^{\pi_k}_1) - \delta \G((\Qmd^k)_1) \,\big|\, \hist_k\right], \\
    \noise_k &\coloneqq \delta \G(p^{\pi_k}_1) - \delta \G((\Qmd^k)_1) - \bias_k.
\end{align}
Here, $\bias_k$ captures systematic approximation error, and $\noise_k$ represents a zero-mean fluctuation conditional on $\hist_k$. Under mild assumptions over noise and bias (see \cref{sec:app-robust-proof}), the long-term behavior of the iterates can be characterized.

\begin{tcolorbox}[colframe=white!, top=2pt,left=2pt,right=2pt,bottom=2pt]
\begin{restatable}[Asymptotic convergence under inexact updates (Informal)]{theorem}{trajGeneralCase}
\label{theorem:general_case_convergence}
\looseness -1 Assume the oracle has bounded variance and diminishing bias, and the step sizes $\{\gamma^k\}$ satisfy the Robbins--Monro conditions
($\sum_k \gamma^k = \infty$, $\sum_k (\gamma^k)^2 < \infty$).
Then the sequence $\{p^{\pi_k}_1\}$ generated by \cref{alg:algorithm} converges almost surely to the optimum in the weak sense:
\begin{equation}
    p^{\pi_k}_1 \rightharpoonup \Tilde{p}_1 \quad \text{a.s.}, 
\end{equation}
where $\Tilde{p}_1$ is a stationary point of $\G$.
\end{restatable}
\end{tcolorbox}
\noindent\textbf{Remark.}
\looseness -1 Several functionals $\mathcal{G}$ considered in this work (\eg forward and reverse KL, Jensen--Shannon, and their $f$-guided counterparts etc.) are convex over the space of measures. In these cases, Thm. \ref{theorem:general_case_convergence} trivially strengthens to convergence to a \emph{global} optimum: $p^{\pi_k}_1 \rightharpoonup p_1^\star=\mathbf{Q}_1^\star$ for some $\mathbf{Q}^\star \in \arg\max_{\mathbf{Q}:\,\mathbf{Q}_0=p_0^{\mathrm{pre}}}\mathcal{G}(\mathbf{Q}_1)$, as shown in Apx. \ref{sec:app-robust-proof}.

\begingroup
  \captionsetup[subfigure]{aboveskip=1.7pt, belowskip=0pt}
\newlength{\imgw}
\setlength{\imgw}{0.25\textwidth}
\begin{figure*}[t]
    \centering
    \begin{subfigure}{\imgw}
      \centering
      \includegraphics[width=\textwidth]{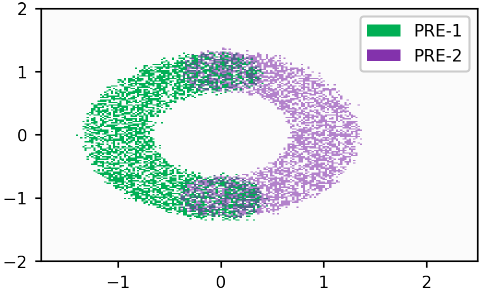}
      \caption{Pre-trained samples}
      \label{fig:toy_top_a}
    \end{subfigure}%
    \begin{subfigure}{\imgw}
      \centering
      \includegraphics[width=\textwidth]{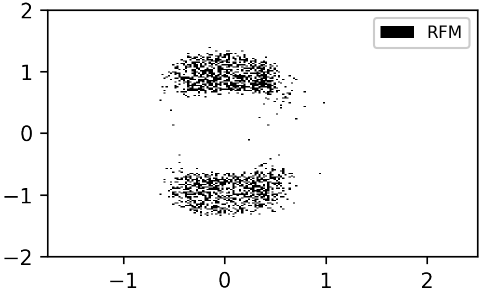}
      \caption{AND Balanced}
      \label{fig:toy_top_b}
    \end{subfigure}%
    \begin{subfigure}{\imgw}
      \centering
      \includegraphics[width=\textwidth]{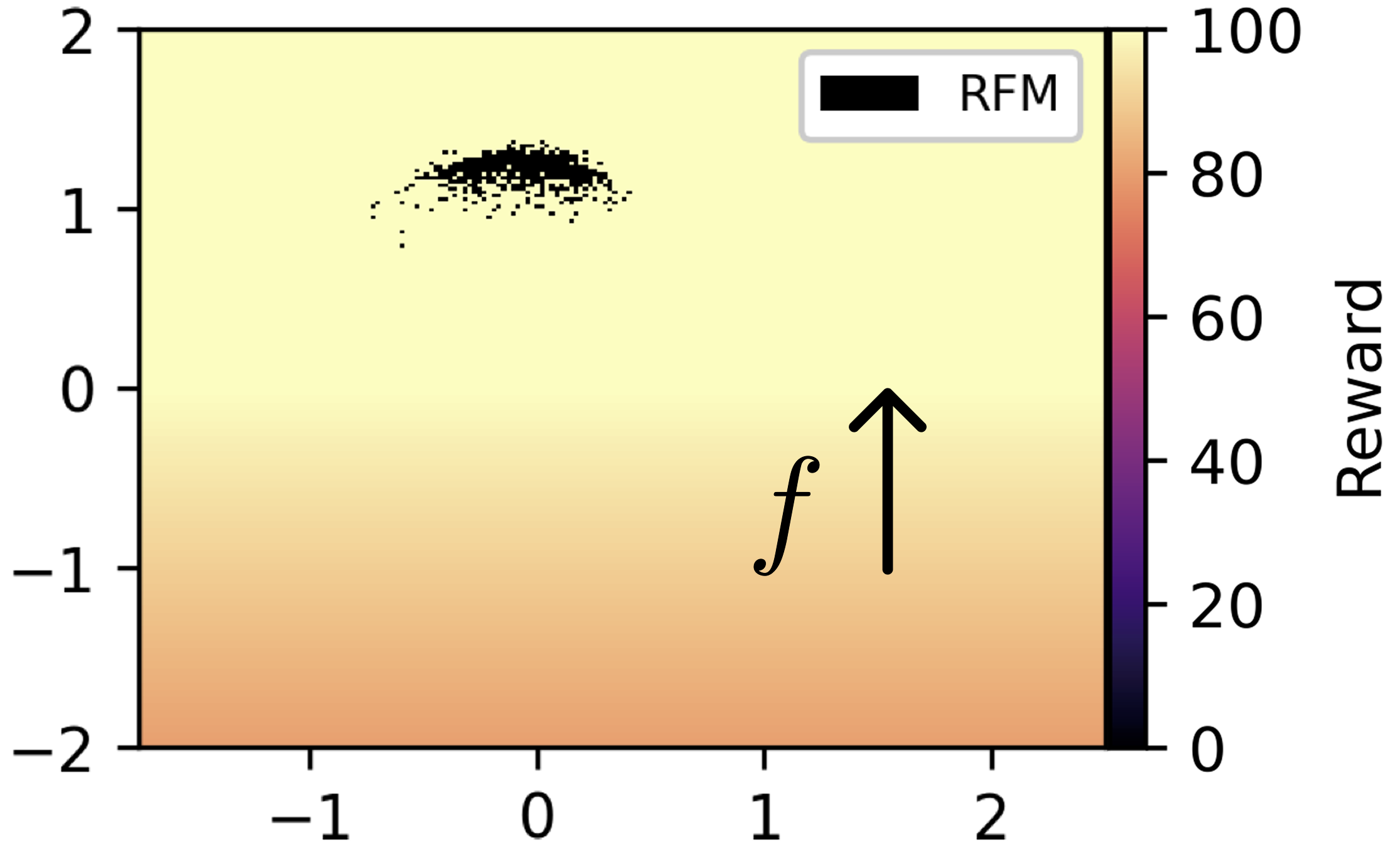}
      \caption{AND reward up}
      \label{fig:toy_top_c}
    \end{subfigure}%
    \begin{subfigure}{\imgw}
      \centering
      \includegraphics[width=\textwidth]{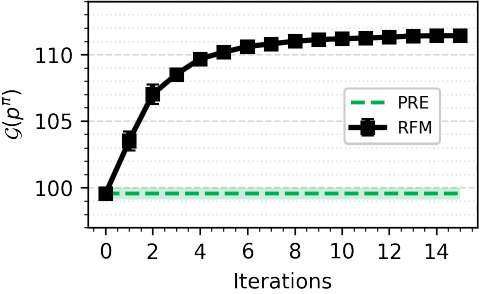}
      \caption{AND reward up}
      \label{fig:toy_top_d}
    \end{subfigure}
    \\[0.4em]
    \begin{subfigure}{\imgw}
      \centering
      \includegraphics[width=\textwidth]{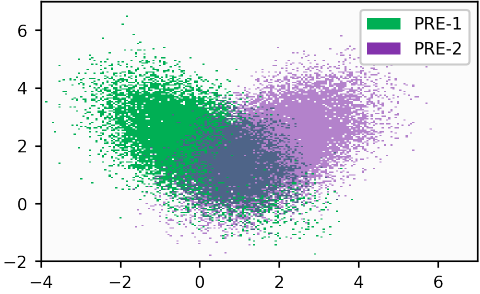}
      \caption{Pre-trained samples}
      \label{fig:toy_mid_a}
    \end{subfigure}%
    \begin{subfigure}{\imgw}
      \centering
      \includegraphics[width=\textwidth]{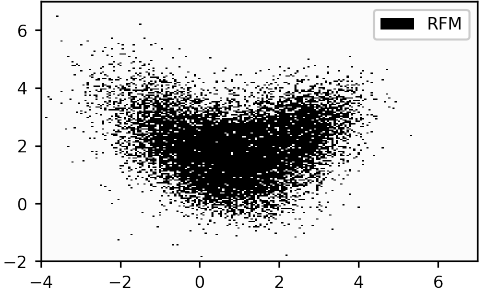}
      \caption{OR Balanced}
      \label{fig:toy_mid_b}
    \end{subfigure}%
    \begin{subfigure}{\imgw}
      \centering
      \includegraphics[width=\textwidth]{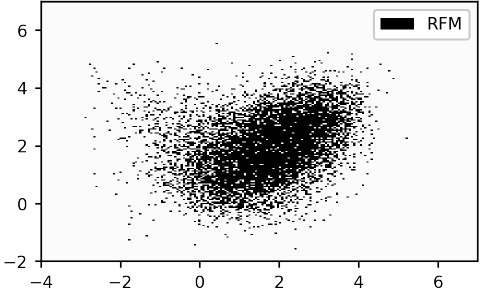}
      \caption{OR $\alpha=[0.1,0.9]$}
      \label{fig:toy_mid_c}
    \end{subfigure}%
    \begin{subfigure}{\imgw}
      \centering
      \includegraphics[width=\textwidth]{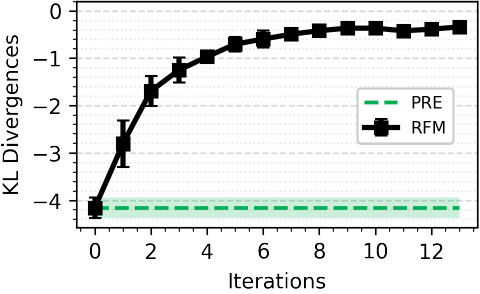}
      \caption{OR optimization}
      \label{fig:toy_mid_d}
    \end{subfigure} 
    \\[0.4em]
    \begin{subfigure}{\imgw}
      \centering
      \includegraphics[width=\textwidth]{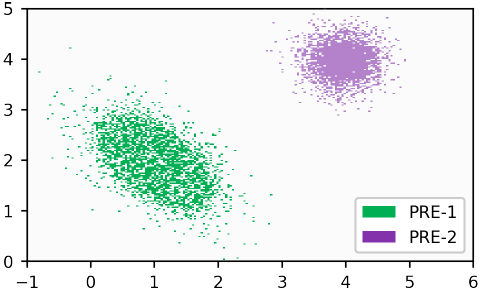}
      \caption{Pre-trained samples}
      \label{fig:toy_bottom_a}
    \end{subfigure}%
    \begin{subfigure}{\imgw}
      \centering
      \includegraphics[width=\textwidth]{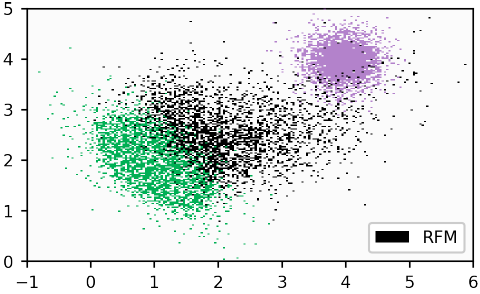}
      \caption{INTR $\pi^{init} = \pi^{pre,1}$}
      \label{fig:toy_bottom_b}
    \end{subfigure}%
    \begin{subfigure}{\imgw}
      \centering
      \includegraphics[width=\textwidth]{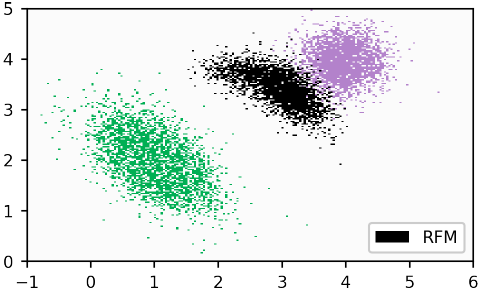}
      \caption{INTR $\pi^{init} = \pi^{pre,2}$}
      \label{fig:toy_bottom_c}
    \end{subfigure}%
    \begin{subfigure}{\imgw}
      \centering
      \includegraphics[width=\textwidth]{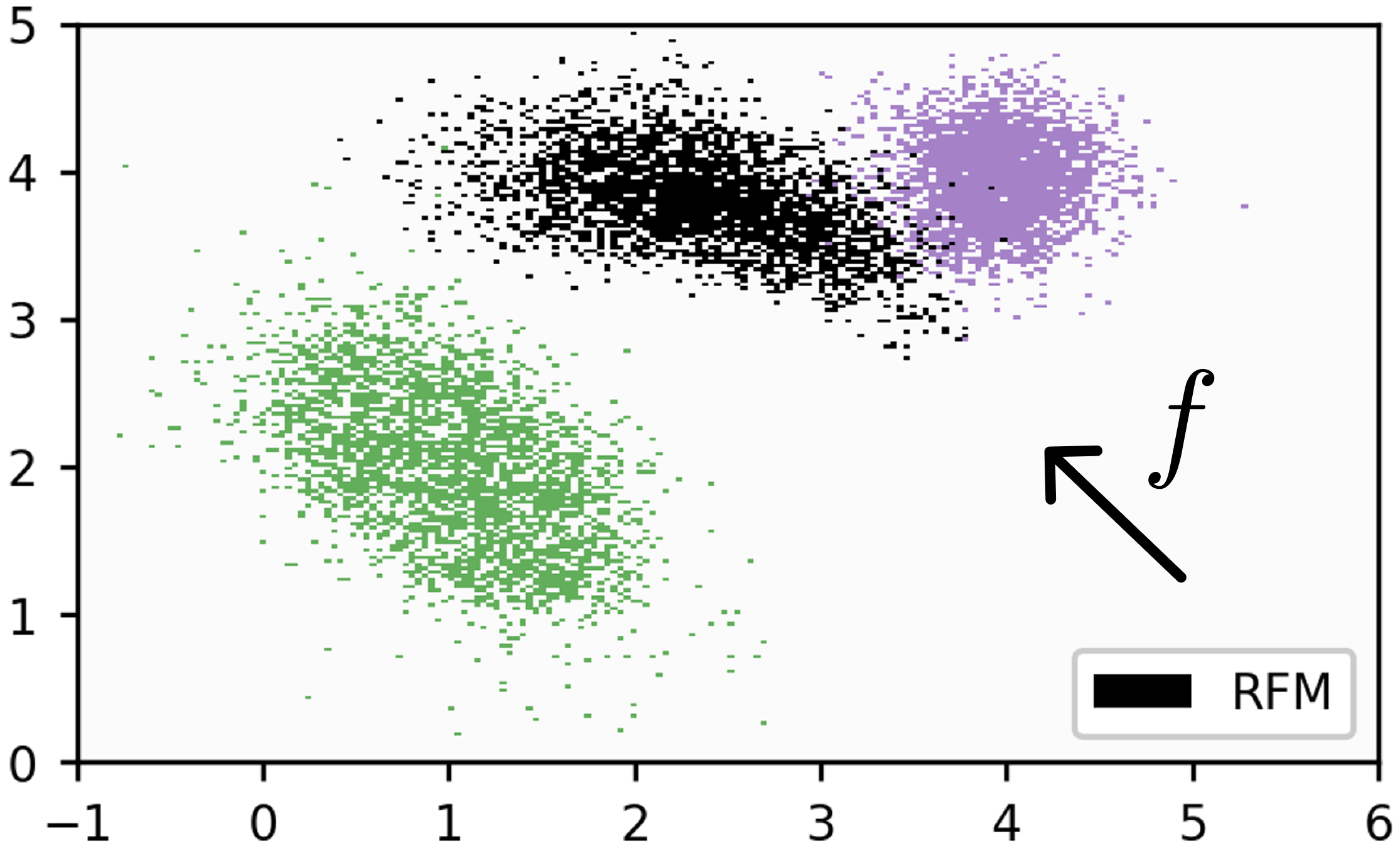}
      \caption{Reward-guided INTR}
      \label{fig:toy_bottom_d}
    \end{subfigure} 
    \caption{\looseness-1 Illustrative settings with visually interpretable results. (top) Flow model balanced pure intersection (\ref{fig:toy_top_b}), and reward-guided intersection (\ref{fig:toy_top_c}), (mid) Flow balanced and unbalanced union, (bottom) Flow model pure and reward-guided interpolation. Crucially, \AlgNameShort can correctly implement these practically relevant and diverse operators with high degree of expressivity (\eg $\alpha$, reward-guidance).} 
    \label{fig:experiments_fig_1}
\end{figure*}
\endgroup

%% file: sections/experiments.tex
\looseness -1 We evaluate \AlgNameShort for the reward-guided flow merging problem (see Eq. \ref{eq:reward_guided_flow_merging_problem}) by tackling two types of experiments: (i) visually interpretable illustrative settings, showcasing the correctness and high expressivity of \AlgNameShort, and (2) high-dimensional molecular design and conformer generation tasks. Further experimental details are reported in Apx. \ref{sec:experimental_details}

\mypar{Intersection Operator $\mathcal{O}_\land$}
\looseness -1 We consider pre-trained flow models inducing densities \smash{$p_1^{pre,1}$} (green) and \smash{$p_1^{pre,2}$} (violet), as shown in Fig. \ref{fig:toy_top_a}. We fine-tune $\pi^{init} \coloneqq \pi^{pre,1}$ via \AlgNameShort to compute the policy $\pi^*$ resulting from diverse intersection operations \smash{$\pi^* = \mathcal{O}_\land(\pi^{pre,1}, \pi^{pre,2})$}. First, in Fig. \ref{fig:toy_top_b}, we show $p^*$ (black) obtained by \AlgNameShort with $\alpha = [0.5, 0.5]$, \ie \emph{balanced} (B). One can notice that the flow model $p^*$ covers mostly the intersecting regions between \smash{$p_1^{pre,1}$} and \smash{$p_1^{pre,2}$} (see Fig. \ref{fig:toy_top_a}). In Fig. \ref{fig:toy_top_c} we report an instance of reward-guided intersection (RG) for a reward function maximized upward. As one can see, \AlgNameShort computes a policy $\pi^*$ placing density over the highest-reward region among the intersecting ones, \ie the top intersecting area. This reward-guided flow merging process is carried out via maximization over $K=15$ iterations of the objective $\G$ illustrated in Fig. \ref{fig:toy_top_d}.

\mypar{Union Operator $\mathcal{O}_\lor$}
\looseness -1 We fine-tune the pre-trained flow model $\pi^{init} = \pi^{pre,1}$ with density illustrated in Fig. \ref{fig:toy_mid_a} (green) via \AlgNameShort to implement balanced (\ie $\alpha = [0.5, 0.5]$ and unbalanced (\ie $\alpha = [0.1, 0.9]$ (UB)) versions of the union operator, namely computing \smash{$\pi^* = \mathcal{O}_\lor(\pi^{pre,1}, \pi^{pre,2})$}. As shown in Fig. \ref{fig:toy_mid_b} and \ref{fig:toy_mid_c} \AlgNameShort can successfully compute optimal policies $\pi^*$ implementing both operators via optimization of the functional $\G$, corresponding to sum of weighted KL-divergences (see Eq. \ref{eq:or_operator_problem}) evaluated for iterations $k \in [K]$ with $K=13$ in Fig. \ref{fig:toy_mid_d}.

\mypar{Interpolation Operator $\Operator_{W_1}$} 
\looseness -1 We use \AlgNameShort to compute flows $\pi^*$ inducing $p_1^*$ corresponding to diverse interpolations between the the pre-trained models' densities illustrated in Fig. \ref{fig:toy_bottom_a}. Although the optimal policy to which \AlgNameShort converges asymptotically is invariant \wrt the initial flow model $\pi^{init}$ chosen for fine-tuning, here we show that this choice can actually be used to control the algorithm execution over few iterations (\ie $K=6$). As one can expect, Fig. \ref{fig:toy_bottom_b} and \ref{fig:toy_bottom_c} show that the result density after $K=6$ iterations is closer to the flow model chosen as $\pi^{init}$, namely $\pi^{pre,1}$ (green) in Fig. \ref{fig:toy_bottom_b} and $\pi^{pre,2}$ (violet) in Fig. \ref{fig:toy_bottom_c}. We illustrate in Fig. \ref{fig:toy_bottom_d} the density (black) obtained via reward-guided interpolation, with a reward function maximized left upwards.

\mypar{Complex Logical Expressions via Generative Circuits}
\looseness -1 We consider $4$ flow models $\{\pi_{pre,i}\}_{i=1}^4$ as in Fig. \ref{fig:toy2_top_a}, which we aim to merge into one flow $\pi^*$ determined by the logical expression $\pi^* = (\pi_1 \land \pi_2) \lor (\pi_3 \land \pi_4)$. We implement the generative circuit in Fig. \ref{fig:toy2_top_d} via sequential use of \AlgNameShort. First, we compute $\pi_5 \coloneqq \mathcal{O}_\land(\pi^{pre,1}, \pi^{pre,2})$ and $\pi_6 \coloneqq \mathcal{O}_\land(\pi^{pre,3}, \pi^{pre,4})$, shown in Fig. \ref{fig:toy2_top_b}, and subsequently $\pi^* \coloneqq \mathcal{O}_\lor(\pi^{pre,3}, \pi^{pre,4})$, as illustrated in Fig. \ref{fig:toy2_top_c}. This illustrative experiments confirms that \AlgNameShort can implement complex logical expressions over generative models via generative circuits, as the simple one just presented.

\textbf{Low-Energy Molecular Design via Flow Merging} 
\looseness -1 Navigating chemical space to discover novel structures with desirable properties is a central goal of data-driven molecular design. A generative model must produce diverse, chemically valid structures that follow specified property profiles and constraints. We base our case study on two FlowMol models $\pi^{pre,1}$ and $\pi^{pre,2}$ ~\citep{dunn2024mixedcontinuouscategoricalflow} pre-trained on GEOM-Drugs ~\citep{axelrod2022geom} with different levels of single-point total energy at GFN1‐xTB level of theory~\citep{friede2024dxtb}, $-14.8$ and $-8.1$ Ha as shown in Fig. \ref{fig:drugs_a}.  We aim to compute a flow model that generates molecules whose total energy 
is likely under both generative models. To this end, we run \AlgNameShort to compute the flow $\pi^*$ returned by the intersection operator (see Eq. \ref{eq:and_operator_problem}), with parameters detailed in Apx. \ref{sec:experimental_details}. 
\begin{wrapfigure}{r}{0.26\textwidth}
  \centering 
  \includegraphics[width=0.26\textwidth]{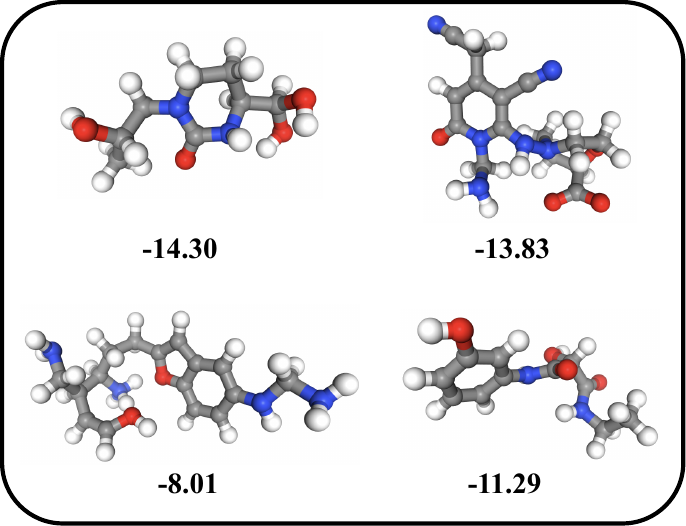}
  \caption{\looseness -1 Drug-like molecules generated by $\pi_{AND}^*$ flow via \AlgNameShort.}
  \label{fig:mols_visual} 
\end{wrapfigure}
We report the density $p_1^*$ (black) computed via balanced merging 
(\ie $\alpha_i = 1$) in Fig. \ref{fig:drugs_b} and the one obtained via unbalanced merging (\ie $\alpha_1=1,\alpha_2=2$) in Fig. \ref{fig:drugs_c}. In the former case, $p_1^*$ correctly places the majority of its density on the overlapping region between the two priors within $[-20,0]$ Ha (see Fig. \ref{fig:drugs_b}). The estimated mean energy of $\pi^*$ (black) \ie $-10.95 \pm 0.28$ Ha, reported along with validity in \ref{fig:drugs_d} matches the energy value of maximal overlap between $\pi^{pre,1}$ and $\pi^{pre,2}$, as one can see in \ref{fig:drugs_a}. Adding reward-guidance leads to lower energy values compared to the balanced merging model while keeping its high validity. 
In the unbalanced case, \AlgNameShort shifts the density slightly leftwards, effectively implementing the $\alpha$-weighted intersection. We report energy-validity metrics resulting from balanced and unbalanced intersection in Fig. \ref{fig:drugs_d}, and compare them with their reward-guided counterpart in Table \ref{tab:mol_exp_stats}. Next, we compute via \AlgNameShort the union operator over two FlowMol pre-trained on the QM9 dataset ~\citep{ramakrishnan2014quantum}. We parametrize critics $\phi_i^*$ (see Sec. \ref{alg:algorithm}) via the FlowMol latent representation with an MLP readout layer. Figure \ref{fig:union_qm9} shows that the estimated mean of the model $\pi^*$ obtained via \AlgNameShort matches the average total energy of $\pi^{pre,1}$ and $\pi^{pre,2}$ as predicted by the closed-form expression for union from Sec. \ref{sec:problem_setting}. 

\begingroup
  \captionsetup[subfigure]{aboveskip=1.7pt, belowskip=0pt}
\setlength{\imgw}{0.25\textwidth}
\begin{figure*}[t]
    \centering
    \begin{subfigure}{\imgw}
      \centering
      \includegraphics[width=\textwidth]{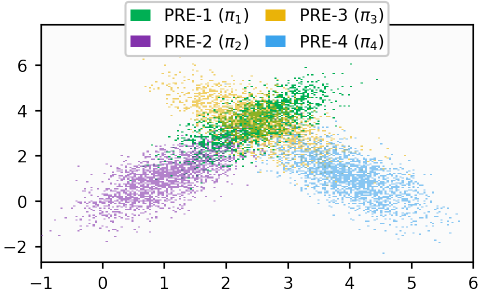}
      \caption{Pre-trained samples}
      \label{fig:toy2_top_a}
    \end{subfigure}%
    \begin{subfigure}{\imgw}
      \centering
      \includegraphics[width=\textwidth]{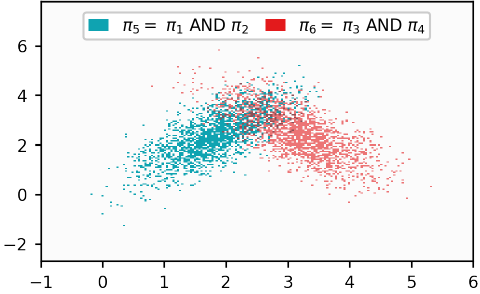}
      \caption{$\pi_5$ and $\pi_6$}
      \label{fig:toy2_top_b}
    \end{subfigure}%
    \begin{subfigure}{\imgw}
      \centering
      \includegraphics[width=\textwidth]{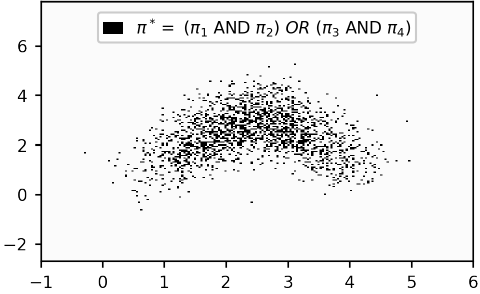}
      \caption{Circuit output $\pi^*$}
      \label{fig:toy2_top_c}
    \end{subfigure}%
    \begin{subfigure}{\imgw}
      \centering
      \raisebox{2ex}[0pt][0pt]{
      \includegraphics[width=0.9\textwidth]{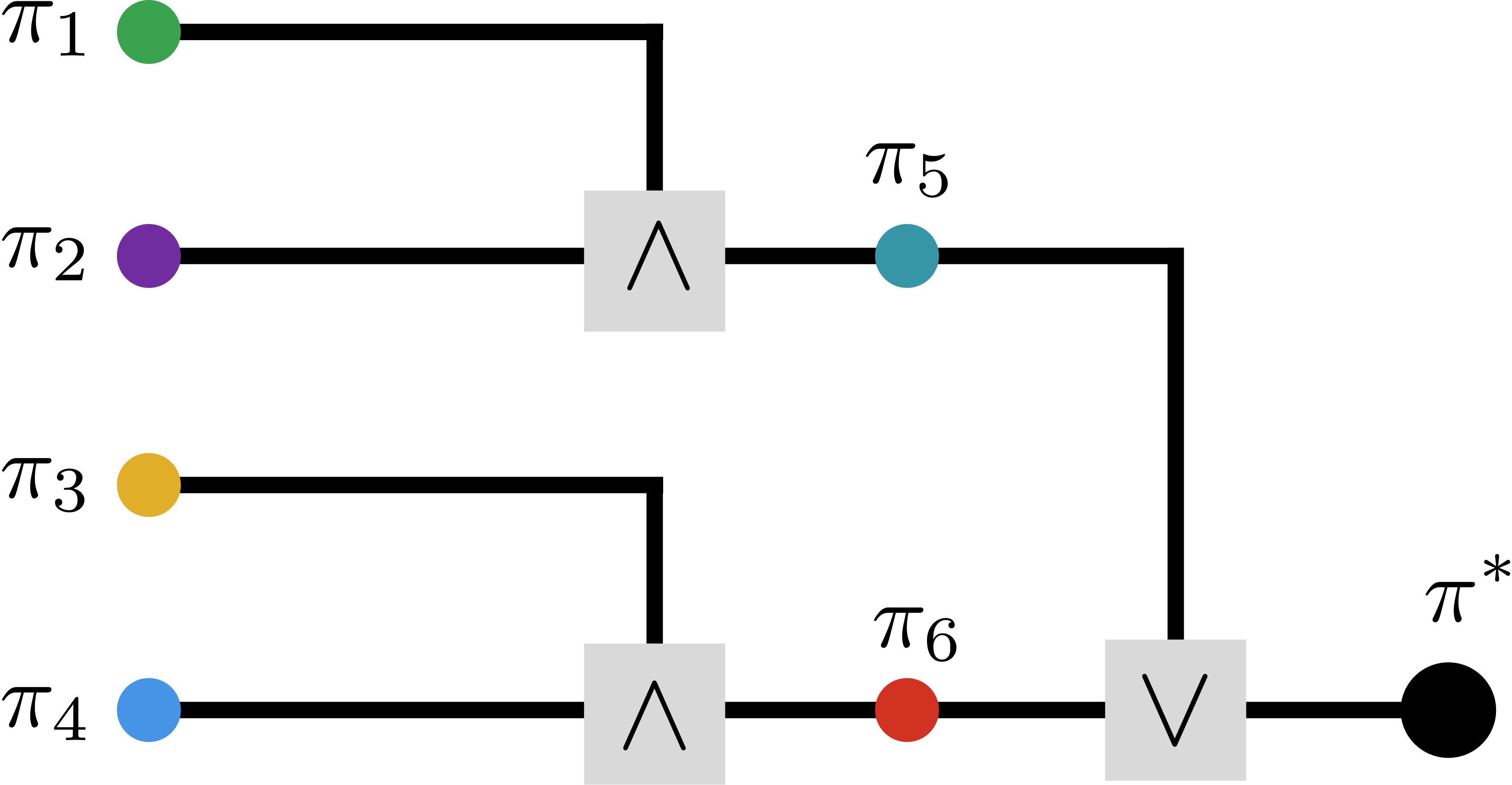}}
      \caption{Generative circuit}
      \label{fig:toy2_top_d}
    \end{subfigure}%
    \\[0.4em]
    \begin{subfigure}{\imgw}
      \centering
      \includegraphics[width=\textwidth]{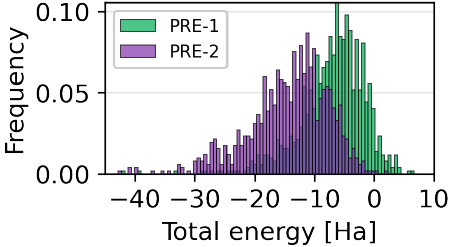}
      \caption{Pre-trained samples}
      \label{fig:drugs_a}
    \end{subfigure}%
    \begin{subfigure}{\imgw}
      \centering
      \includegraphics[width=\textwidth]{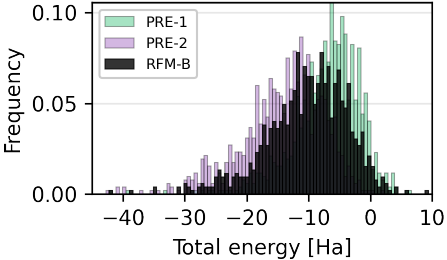}
      \caption{AND molecules}
      \label{fig:drugs_b}
    \end{subfigure}%
    \begin{subfigure}{\imgw}
      \centering
      \includegraphics[width=\textwidth]{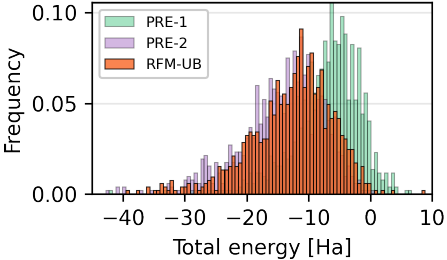}
      \caption{AND $\alpha=[0.33, 0.66]$}
      \label{fig:drugs_c}
    \end{subfigure}%
    \begin{subfigure}{\imgw}
      \centering
      \includegraphics[width=\textwidth]{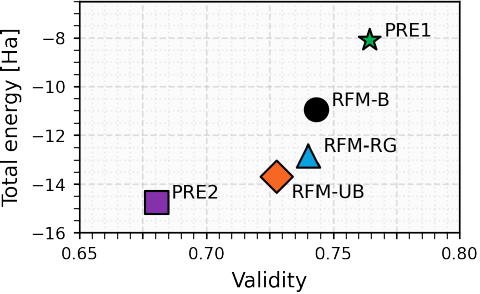}
      \caption{Validity-Energy}
      \label{fig:drugs_d}
    \end{subfigure} 
    \caption{\looseness-1 \looseness -1 (top) \AlgNameShort implements a generative circuit (\ref{fig:toy2_top_d}) describing a complex logical expressions ($\pi^* = (\pi_1 \land \pi_2) \lor (\pi_3 \land \pi_4)$) by computing sequential operators (\ref{fig:toy2_top_a}-\ref{fig:toy2_top_c}). (bottom) \AlgNameShort computes a flows intersection $\pi^*$ generating drug molecules with desired energy levels. } 
    \label{fig:experiments_fig_2} 
\end{figure*}
\endgroup

\begingroup
  \captionsetup[subfigure]{aboveskip=1.7pt, belowskip=0pt}
\setlength{\imgw}{0.25\textwidth}
\begin{figure*}[t]
    \centering
    \begin{subfigure}{\imgw}
      \centering
      \includegraphics[width=\textwidth]{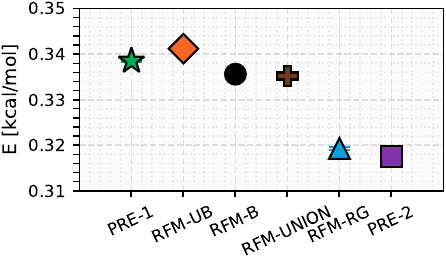}
      \caption{Energy $E$}
      \label{fig:conf_energy}
    \end{subfigure}%
    \begin{subfigure}{\imgw}
      \centering
      \includegraphics[width=\textwidth]{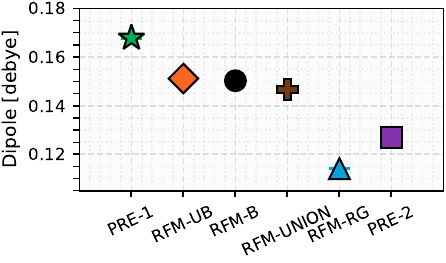}
      \caption{Dipole moment $\mu$}
      \label{fig:conf_dipole}
    \end{subfigure}%
    \begin{subfigure}{\imgw}
      \centering
      \includegraphics[width=\textwidth]{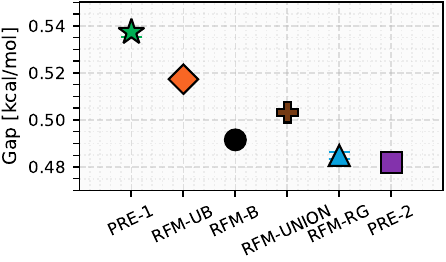}
      \caption{HOMO-LUMO gap $\Delta \epsilon$}
      \label{fig:conf_eps}
    \end{subfigure}%
    \begin{subfigure}{\imgw}
      \centering
      \includegraphics[width=\textwidth]{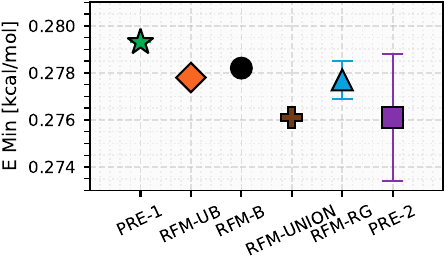}
      \caption{Min energy $E_{min}$}
      \label{fig:conf_min_energy}
    \end{subfigure}%
    \\[0.4em] 
    \caption{\looseness-1 \AlgNameShort can perform balanced (B), unbalanced (UB), reward-guided (RG) intersections, as well as unions (UNION) of ETFlow~\citep{hassan2024etflow} conformer generation models. We evaluate the resulting flows in terms of median absolute errors of energy (\ref{fig:conf_energy}), dipole moment (\ref{fig:conf_dipole}), HOMO–LUMO gap (\ref{fig:conf_eps}), and minimum energy (\ref{fig:conf_min_energy}). These results demonstrate the ability of \AlgNameShort to compute new flow models whose properties predictably interpolate those of the available pre-trained flows.\\
    } 
    \label{fig:experiments_fig_3} 
\end{figure*}
\endgroup

\textbf{Reward-Guided FM of Conformer Generation Models} 
\looseness -1 
Inferring 3D conformers from a molecule's topology is a key prerequisite for many computational chemistry applications including molecular docking~\citep{mcnutt2023conformer}, thermodynamic property prediction~\citep{pracht2021conformer}, and modeling reaction pathways for catalyst design~\citep{schmid2025rapid}. Given two prior ETFlow models (\ie PRE-1 and PRE-2) \citep{hassan2024etflow} with different property profiles, and we aim to merge them into a conformer generator whose profiles controllably interpolate between, or slightly improve upon their initial energetic ensemble property profile. In particular, we evaluate errors in energy, dipole moment, HOMO-LUMO gap and minimum energy of the generated structure ensemble compared to the equilibrium ensemble.

\looseness -1 We run \AlgNameShort initialized from PRE-2, to compute its balanced (B), unbalanced (UB), reward-guided (RG) intersection, and union variants. Figure~\ref{fig:conf_energy} shows that the median absolute error (MAE) on the total energy $E$ interpolates between PRE-1 and PRE-2, \AlgNameShort-B and \AlgNameShort-UNION achieve intermediate errors of $\approx$ 0.3356 and 0.3352 ~kcal/mol as expected. On the other hand, the reward-guided (\ie via energy minimization) counterpart reaches lower energy values, namely $\approx$ 0.3193~kcal/mol, and the unbalanced variant ($\alpha_1=0.7, \alpha_2=0.3$) remains near PRE-1 at 0.3412~kcal/mol. These results validate the ability of \AlgNameShort to perform unbalanced and reward-guided intersection. We report similar results for the dipole moment $\mu$ in Fig.~\ref{fig:conf_dipole}, for the HOMO–LUMO gap $\Delta\epsilon$~\ref{fig:conf_eps}, and minimum energy $E_{\min}$~\ref{fig:conf_min_energy}. Our evaluation indicates that \AlgNameShort can perform (reward-guided) flow merging with conformer generation models, leading to flows with controllable interpolations or improvements over property profiles of pre-trained models. 

\looseness -1 Moreover, in Apx.~\ref{sec:app_computational_cost}, we briefly investigate the computational cost of \AlgNameLong. Ultimately, although this work is primarily motivated by scientific discovery problems, we report in Apx. \ref{sec:image_experiments} an illustrative application of \AlgNameShort to image-generation diffusion models.

%% file: sections/related_works.tex
\mypar{Flow models fine-tuning via optimal control}
\looseness -1 Several works have framed fine-tuning of flow and diffusion models to maximize expected reward functions under KL regularization as an entropy-regularized optimal control problem~\citep[\eg][]{uehara2024fine, tang2024fine, uehara2024feedback,  domingo2024adjoint, gutjahr2025constrained}. More recently, \citet{santi2025flow} introduced a framework for distributional fine-tuning. The problem tackled in this work (see Eq. \ref{eq:reward_guided_flow_merging_problem}) instead extends the orthogonal setting of  expected rewards with arbitrary divergences to the case with $n>1$ pre-trained models. This generalization ($i$) unifies flow control and merging, ($ii$) renders possible to use of scalable control-based or RL schemes~\citep[\eg][]{domingo2024adjoint} for flow merging, and ($iii$) enables reward-guided flow merging.

\mypar{Flow model merging and inference-time composition}
\looseness -1 Recent works in inference-time flow and diffusion model composition introduced theory-backed schemes~\citep[\eg][]{skreta2024superposition, bradley2025mechanisms, du2023reduce}. On the other hand, our work tackles the problem of (reward-guided) flow merging~\citep[\eg][]{song2023consistency}, a significantly less explored research problem. Crucially, while inference-time flow composition aims to compose models at sampling time, flow merging aims to combine multiple flow models into one, and then disregard the prior models. This work provides a formal probability-space viewpoint on the latter problem, introduces interpretable merging operators (see Sec. \ref{sec:problem_setting}) for highly expressive compositions (\eg via generative circuits), provably implemented by \AlgNameShort, which is to our knowledge the first scheme for provable reward-guided flow merging. Moreover, to our knowledge, the theoretical guarantees in Sec. \ref{sec:theory} are first-of-their-kind for merging of flow and diffusion models. In particular, specializing them to specific operators \eg intersection, yields highly relevant results, such as generative models safety guarantees via intersection with prior safe models.

\mypar{Convex and general utilities reinforcement learning}
\looseness -1 Convex and General (Utilities) RL~\citep{hazan2019maxent, zahavy2021reward, zhang2020variational} generalizes RL to maximization of a concave~\citep{hazan2019maxent, zahavy2021reward} or general~\citep{zhang2020variational, barakat2023reinforcement} functional of the state distribution induced by a policy over a dynamical system's state space. 
Recent works tackled the finite samples budget setting~\citep[\eg][]{mutti2022importance, mutti2022challenging, mutti2023convex, de2024global, de2024geometric}. Similarly to previous optimization schemes for diffusion models~\citep{de2025provable, santi2025flow}, our framework (in Eq. \ref{eq:reward_guided_flow_merging_problem}) is related to Convex and General RL, with $p^\pi_1$ being the state distribution induced by policy $\pi$ over a subset of the flow process state space. 

\mypar{Optimization over probability measures via mirror flows}
\looseness -1 Recently, there has been a growing interest in devising theoretical guarantees for probability-space optimization problems in diverse fields of application. These include optimal transport~\citep{aubin2022mirror, leger2021gradient, karimi2024sinkhorn}, kernelized methods~\citep{dvurechensky2024analysis}, GANs~\citep{hsieh2019finding}, and manifold exploration~\citep{de2025provable} among others. To our knowledge, we present the first use of this theoretical framework to establish guarantees for flow and diffusion models merging.

%% file: sections/conclusions.tex
\looseness -1 We introduce a probability-space optimization framework for reward-guided flow merging, unifying and generalizing existing formulations. This allows to express diverse practically relevant operators over generative model densities (\eg intersection, union, interpolation, logical expressions, and their reward-guided counterparts). We propose \AlgNameLong, a mirror-descent scheme reducing complex merging tasks to sequential standard fine-tuning steps, solvable by established methods. Leveraging advances in mirror flows theory, we provide first-of-their kind guarantees for (reward-guided) flow merging. Empirical results on interpretable settings, molecular design, and conformer generation tasks demonstrate that our approach can steer pre-trained models to implement reward-guided merging tasks of high practical relevance.

%% file: appendices/app_theory2.tex
\newcommand{\Qbf}{\mathbf{Q}}

\subsection{Proof of Theorem \ref{theorem:AMretainsScores}}
\paragraph{Stochastic Optimal Control.} 
We consider stochastic optimal control (SOC), which studies the problem of steering a stochastic dynamical system to optimize a specified performance criterion. 
Formally, let $(X_t^u)_{t\in[0,1]}$ be a controlled stochastic process satisfying the stochastic differential equation (SDE)
\begin{equation}
    \drm X_t^u = b(X_t^u, t)\, \drm t + \sigma(t)\, u(X_t^u, t)\, \drm t + \sigma(t)\, \drm B_t, 
    \qquad X_0^u \sim p_0, \label{eq:SOC}
\end{equation}
where $u \in \mathcal{U}$ is an admissible control and $B_t$ is standard Brownian motion. 
The objective is to select $u$ to minimize the cost functional
\begin{equation}
\label{eq:SOC_app}
     \EV \Bigg[ \int_0^1 \frac{1}{2}\|u(X_t^u,t) \|^2 \, dt - g(X_1^u) \Bigg],
\end{equation}
where $\frac{1}{2}\|u(\cdot,t) \|^2$ represents the running cost and $g$ is a terminal reward. 
A standard application of Girsanov's theorem shows that \cref{eq:SOC_app} is equivalent to the mirror descent iterate in \cref{eq:MD_step_copy} with $\delta \G(p_1^{\pi_k}) \leftarrow g$ and $p_0 \leftarrow p^{pre}$~\citep{tang2024fine}. In addition, it is well-known that in the context of diffusion-based generative modeling, the corresponding uncontrolled dynamics
\begin{equation}
    \drm X_t = -b(X_t,t)\, \drm t + \sigma(t)\, \drm B_t \label{eq:forward}
\end{equation}
coincide with the forward noising process used in score-based models~\citep{song2021scorebased,domingo2024adjoint}.

\paragraph{Proof of Theorem \ref{theorem:AMretainsScores}.} 
\AMretainsScores*

\begin{proof}

\textbf{Step 1.}  
Let $\mathbf{Q}^\star$ denote the optimal process solving \cref{eq:SOC}.  
A standard application of Girsanov's theorem shows that $\mathbf{Q}^\star$ also solves the \emph{Schrödinger bridge problem}  
\begin{equation}
\label{eq:SBP}
\min_{\substack{\mathbf{Q}_0 = p^{\text{pre}} \\ \mathbf{Q}_1 = \mathbf{Q}^\star_1}} 
\KL\!\big(\mathbf{Q} \,\|\, \mathbf{P}\big),
\end{equation}  
where $\mathbf{P}$ is the law of the uncontrolled dynamics
\[
\drm X_t = b(X_t,t)\, \drm t + \sigma(t)\, \drm B_t.
\]  
This equivalence holds because the SOC cost in \cref{eq:SOC} penalizes control energy in the same way that Girsanov’s theorem expresses a controlled SDE as a relative entropy with respect to its uncontrolled counterpart.

\textbf{Step 2.}  
Define the \emph{forward process} $\mathbf{P}_{\text{forward}}$ by
\begin{equation}
\label{eq:forward_process}
\drm X_t = -b(X_t,t)\, \drm t + \sigma(t)\, \drm B_t.
\end{equation}  
By assumption, this process maps any initial distribution to the standard Gaussian at $t=1$. In particular, starting from $X_0 \sim \Qbf^\star_1$, we obtain $X_1 \sim p^{\text{pre}} = \mathcal{N}(0,I)$.

\textbf{Step 3.}  
Consider the \emph{time-reversed Schrödinger bridge problem}
\begin{equation}
\label{eq:SBP-reverse}
\min_{\substack{\overleftarrow{\mathbf{Q}}_0 = \mathbf{Q}^\star_1\\ \overleftarrow{\mathbf{Q}}_1 = p^{\text{pre}}}} 
\KL\!\big(\overleftarrow{\mathbf{Q}} \,\|\, \mathbf{P}_{\text{forward}}\big),
\end{equation}  
and denote its solution by $\overleftarrow{\mathbf{Q}}^\star$.  
Since relative entropy is invariant under bijective mappings and time-reversal is bijective, the optimizers of \cref{eq:SBP} and \cref{eq:SBP-reverse} satisfy
\[
\overleftarrow{\mathbf{Q}}^\star \;=\;\overleftarrow{\hspace{0.12cm}\mathbf{Q}^\star\hspace{0.1cm}}
\]
\ie the optimal reversed bridge is simply the time-reversal of the forward bridge.

By \textbf{Step 2}, the process
\begin{equation}
\label{eq:SB-sol}
\drm X_t = -b(X_t,t)\, \drm t + \sigma(t)\, \drm B_t, 
\qquad X_0 \sim \Qbf^\star_1
\end{equation}
solves \cref{eq:SBP-reverse}, achieving the minimum relative entropy (zero) while satisfying the prescribed marginals.  
Thus, invoking the relation $\overleftarrow{\mathbf{Q}}^\star = \overleftarrow{\mathbf{Q}^\star}$, the solution to \cref{eq:SBP}—and hence to the SOC problem \cref{eq:SOC}—is given by the time-reversal of \cref{eq:SB-sol}.

Finally, applying the classical time-reversal formula \citep{anderson1982reverse} yields that $\Qbf^\star$ is given by
\[
\drm {X}_t 
= \Big( b(\overleftarrow{X}_t,t) + \sigma^2(t)\, \nabla \log p_t({X}_t) \Big)\, \drm t 
+ \sigma(t)\, \drm {B}_t,
\]
where $p_t$ is the marginal density of \cref{eq:SB-sol}. Hence, \RewardGuidedFineTuningSolver exactly recovers the score function.
\end{proof}

\subsection{Rigorous Statement and Proof of Theorem \ref{theorem:general_case_convergence}}
\label{sec:app-robust-proof}

\renewcommand{\Q}{\entropy}

To prepare for the convergence analysis, we impose a few auxiliary assumptions. These assumptions are standard in the study of stochastic approximation and gradient flows, and typically hold in practical situations. Our proof strategy follows ideas that have also been employed in related works \citep{de2025provable,santi2025flow}.

We begin with the entropy functional defined on probability measures:
\begin{equation}
\label{eq:relative_properties_rewritten}
\Q(p) \coloneqq \int p \log p .
\end{equation}
In our analysis, $\Q$ serves as the \emph{mirror map} or \emph{distance-generating function} \citep{mertikopoulos2024unified,hsieh2019finding}. The first condition addresses the behavior of the corresponding dual variables.  

\begin{assumption}[Precompactness of Dual Iterates]
\label{asm:precompact_rewritten}
The sequence of dual elements $\{\delta \Q(p_1^{\pi_k})\}_k$ is precompact in the $L_\infty$ topology.
\end{assumption}

\noindent
This compactness property ensures that the interpolated dual trajectories remain confined to a bounded region of function space. Such a condition is crucial for invoking convergence results based on asymptotic pseudotrajectories. Variants of this assumption have appeared in the literature on stochastic approximation and continuous-time embeddings of discrete algorithms \citep{benaim2006dynamics,hsieh2019finding,mertikopoulos2024unified}.  

\begin{assumption}[Noise and Bias Conditions]
\label{asm:approximate_rewritten}
For the stochastic approximations used in the updates, we assume that almost surely:
\begin{align}
   &\|\bias_k\|_\infty \to 0, \\
   &\sum_{k} \EV\!\left[\step_k^2 \big(\|\bias_k\|_\infty^2 + \|\noise_k\|_\infty^2\big)\right] < \infty, \\
   &\sum_{k} \step_k \|\bias_k\|_\infty < \infty.
   \label{eq:bias-step-rewritten}
\end{align}
\end{assumption}

\noindent
These conditions, standard in the Robbins--Monro setting \citep{robbins1951stochastic,benaim2006dynamics,hsieh2019finding}, guarantee that the stochastic bias vanishes asymptotically while the cumulative noise remains under control. Together, they ensure that random perturbations do not obstruct convergence to the optimizer of the limiting objective.  

With these assumptions in place, we can now state and prove the convergence guarantee.  

\begin{tcolorbox}[colframe=white!, top=2pt,left=2pt,right=2pt,bottom=2pt]
\begin{restatable}[Convergence guarantee in the trajectory setting]{theorem}{trajGeneralCase_rigorous_rewritten}
\label{theorem:general_case_convergence_rewritten}
Suppose Assumptions \ref{asm:precompact_rewritten}--\ref{asm:approximate_rewritten} hold, and the step sizes $\{\gamma_k\}$ follow the Robbins--Monro conditions 
($\sum_k \gamma_k = \infty$, $\sum_k \gamma_k^2 < \infty$).  
Then the sequence $\{p^{\pi_k}_1\}$ generated by \cref{alg:algorithm} converges almost surely, in the weak topology, to the optimum:
\begin{equation}
    p^{\pi_k}_1 \rightharpoonup p_1^* \quad \text{a.s.},
\end{equation}
where $p_1^* = \mathbf{Q}^*_1$ for some $\mathbf{Q}^* \in \argmax_{\mathbf{Q}:\mathbf{Q}_0 = p_0^{pre}} \mathcal{G}(\mathbf{Q}_1)$.
\end{restatable}
\end{tcolorbox}

\begin{proof}
We analyze the continuous-time mirror flow defined by
\begin{equation}\label{eq:MF_rewritten}
\dot{h}_t = \delta \mathcal{G}(p^t_1), 
\qquad 
p^t_1 = \delta \Q^\star(h_t),
\end{equation}
where the Fenchel conjugate of $\Q$ is given by 
$\Q^\star(h) = \log \int e^{h}$ \citep{hsieh2019finding,hiriart2004fundamentals}.  

To link the discrete dynamics to this continuous flow, we construct a piecewise linear interpolation of the iterates:
\[
\hat{h}_t = h^{(k)} + \frac{t - \tau_k}{\tau_{k+1} - \tau_k}\big(h^{(k+1)} - h^{(k)}\big), 
\quad 
h^{(k)} = \delta \Q(p^{\pi_k}_1), 
\quad 
\tau_k = \sum_{r=0}^k \alpha_r,
\]
where $\{\alpha_r\}$ denotes the step-size sequence. This interpolation produces a continuous path $\hat{h}_t$ that tracks the discrete updates as the steps shrink.  

Let $\Phi_u$ denote the flow map of \eqref{eq:MF_rewritten} at time $u$. Standard results in stochastic approximation \citep{benaim2006dynamics,hsieh2019finding,mertikopoulos2024unified} imply that for any fixed horizon $T>0$, there exists a constant $C(T)$ such that
\[
\sup_{0 \leq u \leq T} 
\|\hat{h}_{t+u} - \Phi_u(\hat{h}_t)\|
\leq 
C(T)\Big[\Delta(t-1,T+1) + b(T) + \gamma(T)\Big],
\]
where $\Delta$ accounts for cumulative noise, $b$ for bias, and $\gamma$ for step-size effects. Under Assumptions \ref{asm:precompact_rewritten}--\ref{asm:approximate_rewritten}, these quantities vanish asymptotically, ensuring that $\hat{h}_t$ forms a precompact asymptotic pseudotrajectory (APT) of the mirror flow.  

By the APT limit set theorem \citep[Thm.~4.2]{benaim2006dynamics}, the limit set of a precompact APT is contained in the internally chain transitive (ICT) set of the underlying flow. In our case, \cref{eq:MF_rewritten} corresponds to a gradient-like flow in the Hellinger--Kantorovich geometry \citep{mielke2025hellinger}, with $\mathcal{G}$ serving as a strict Lyapunov function. As $\mathcal{G}$ decreases strictly along non-stationary trajectories, the ICT set reduces to the collection of stationary points of $\mathcal{G}$.  

Ultimately, notice that if $\mathcal{G}$ is composed of convex divergences (e.g., forward or reverse KL terms) possibly together with a linear component (\ie any reward function $f$), its stationary points coincide with its global maximizers. Consequently, $\hat{h}_t$ converges almost surely to the set of maximizers of $\mathcal{G}$, which establishes the claim.  Moreover, notice that within Sec. \ref{sec:theory}, we report the previously proved statement for stationary-points and then specialize it for global maximizers in the case of convex functionals.
\end{proof}

%% file: appendices/gradients_derivations.tex
\subsection{A brief tutorial on first variation derivation}
In this work, we focus on the functionals that are Fréchet differentiable: Let $V$ be a normed spaces. Consider a functional $F:V \rightarrow \R$.
    There exists a linear operator $A : V \rightarrow \R$ such that the following limit holds
    \begin{equation} \label{eqn_frechet_derivative_def}
        \lim_{\|h\|_V\rightarrow 0} \frac{|F(f + h) - F(f) - A[h]|}{\|h\|_V} = 0.
    \end{equation}
We further assume that $V$ has enough structure such that every element of its dual (the space of bounded linear operator on $V$) admits a compact representation. For example, if $V$ is the space of bounded continuous functions with compact support, there exists a unique positive Borel measure $\mu$ with the same support, which can be identified as the linear functional.
We denote this element as $\delta F[f]$ such that $\langle \delta F[f], h\rangle = A[h]$. Sometimes we also denote it as $\frac{\delta F}{\delta f}$. We will refer to $\delta F[f]$ as the first-order variation of $F$ at $f$.

In the following, we briefly present standard strategies to derive the first-order variation of two broad classes of functionals, including a wide variety of divergence measures, which can be employ to implement novel operators by Eq. \ref{eq:reward_guided_flow_merging_problem}. We consider: $(i)$ those defined in closed form with respect to the density (e.g., forward KL) and, $(ii)$ those defined via variational formulations (e.g.,  Wasserstein distance, reverse KL, and MMD).
\begin{itemize}[leftmargin=*]
    \item \textbf{Category 1: Functional defined in a closed form with respect to the density.} For this class of functionals, the first-order variations can typically be computed using its definition and chain rule.

    Recalling the definition of first variation (\ref{eqn_frechet_derivative_def}), we can calculate the first-order variation of the mean functional, as a trivial example. Given a continuous and bounded function $r: \R^d \rightarrow \R$ and a probability measure $\mu$ on $\R^d$,  define the functional $F(\mu) = \int r(x) \mu(x) dx$. Then we have:
    \begin{equation}
        |F(\mu + \delta \mu) - F(\mu) - \langle r,  \delta\mu \rangle| = 0.
    \end{equation}
    Therefore we obtain that: $\delta F[\mu] = r$ for all $\mu$.
    In the following section, we compute similarly the first variation of the KL divergence.
    \item \textbf{Category 2: Functionals defined through a variational formulation.} Another fundamental subclass of functionals that plays a central role in this work is the one of functionals defined via a variational problem
    \begin{equation}
        F[f] = \sup_{g \in \Omega} G[f, g],
    \end{equation}
    where $\Omega$ is a set of functions or vectors independent of the choice of $f$, and $g$ is optimized over the set $\Omega$. We will assume that the maximizer $g^*(f)$ that reaches the optimal value for $G[f, \cdot]$ is unique (which is the case for the functionals considered in this project).    
    It is known that one can use the Danskin's theorem (also known as the envelope theorem) to compute
    \begin{equation}
        \frac{\delta F[f]}{\delta f} = \partial_f G[f, g^*(f)],
    \end{equation}
    under the assumption that $F$ is differentiable \citep{milgrom2002envelope}.
\end{itemize}

\subsection{Derivation of First Variations used in Sec. \ref{sec:algorithm}}
In the following, we derive explicitly the first variations employed in Sec. \ref{alg:algorithm} 

\begin{itemize}[leftmargin=*]
      \item \textbf{Optimal transport and Wasserstein-p distance (Category 2)}
    Consider the optimal transport problem
    \begin{equation}
        \mathrm{OT}_c(u, v) = \inf_{\gamma}\left\{\int\int c(x, y) d\gamma(x, y): \int \gamma(x, y) dx = u(y), \int \gamma(x, y) d y = v(x)\right\}
    \end{equation}
    where 
    \begin{equation*}
        \Gamma = \left\{ \gamma : \int \gamma(x, y) dx = u(y), \int \gamma(x, y) d y = v(x) \right\}
    \end{equation*}
    It admits the following equivalent dual formulation
    \begin{align}
        \mathrm{OT}_c(u, v) = \sup_{f, g} \left\{\int f du + \int g dv: f(x) + g(y) \leq c(x, y) \right\} 
    \end{align}
    By taking $c(x, y) = \|x-y\|^p$, we recover $\mathrm{OT}_c(u, v) = W_p(u, v)^p$.
    Let $\phi^*$ and $g^*$ be the solution to the above dual optimization problem.
    From the Danskin's theorem, we have
    \begin{equation} \label{eqn_first_order_variation_Wp}
        \frac{\delta}{\delta u} W_p(u, v)^p = \phi^*.
    \end{equation}
    In the special case of $p=1$, we know that $g^* = - \phi^*$ (note that the constraint can be equivalently written as $\|\nabla \phi\| \leq 1$), in which case $\phi^*$ is typically known as the critic in the Wasserstein-GAN framework~\citep[cf. ][]{arjovsky2017wassersteingan}.
    \item \textbf{Reverse KL divergence (Category 2)} We use the variational (Fenchel–Legendre) representation of the forward KL, $D_{KL}(p \| q)$, as in f-GAN~\citep{nowozin2016f}:
    \begin{equation}
        D_{KL}(p \| q) = \sup_{\phi:\X \to \R} \left\{ \EV_p \phi(x) - \EV_q e^{\phi(x)-1} \right\}
    \end{equation}
    which follows from the general f-divergence dual generator $f(u) = u \log u - u + 1$ whose conjugate is $f^*(t) = e^{t-1}$. For fixed $p$ and variable $q$, we define:
    \begin{equation}
        G(q, \phi) \coloneqq \EV_p \phi(x) - \EV_q e^{\phi(x) -1}
    \end{equation}
    Assuming uniqueness of a maximizer $\phi^*(p,q)$, Danskin's (or envelope) theorem yields the first variation  by differentiating $G$ at $\phi^*$:
    \begin{equation}
        \frac{\delta}{\delta q(x)} D_{KL}(p \| q) =  \frac{\delta}{\delta q(x)} \left( - \int q(x) e^{\phi^*(x)-1} \der u \right) = - e^{\phi^*(x)-1} \label{eq:first_var_reverse_KL}
    \end{equation}
    \item \textbf{KL divergence (Category 1)} Consider the KL functional:
    \begin{equation}
        D_{KL}(p \| q) = -\int p \log \frac{p}{q} , \der x
    \end{equation}
    By the definition of the first-order variation (see Eq. \ref{eqn_frechet_derivative_def}), we have:
    \begin{equation}
        \delta D_{KL}(p \| q) = \log \frac{p}{q} + 1
    \end{equation}
\end{itemize}

%% file: appendices/proof_proposition_union.tex
\looseness -1 For the union operator, gradients are defined via critics $\{\phi^*_i\}_{i=1}^n$ learned with the standard variational form of reverse KL, as in f-GAN training of neural samplers~\citep{nowozin2016f}. For $W_1$ interpolation, each $\phi^*_i$ plays the role of a Wasserstein-GAN discriminator with established learning procedures~\citep{arjovsky2017wassersteingan}. In both cases, each critic compares the fine-tuned density to a prior density \smash{$p_1^{pre,i}$}, seemingly requiring one critic per prior. We prove that, surprisingly, this is unnecessary for the union operator, and conjecture that analogous results hold for other divergences.

\UnionMixture*
Prop. \ref{proposition:union_operator_mixture}, which is proved in the following, implies that the union operator in Eq. \ref{eq:or_operator_problem} over $n$ prior models can be implemented by learning a single critic $\phi^*$, as shown in Sec. \ref{sec:experiments}.

\begin{proof}
    We prove the statement for $n=2$, which trivially generalizes to any $n$.
    We first rewrite the LHS optimization problem as:
    \begin{equation}
        \argmin_\pi \F(p^\pi)
    \end{equation}
    where we denote $p_1^\pi$ by $p^\pi$ for notational concision and define $p_1 = p^{pre,i}$ and $p_2 = p^{pre,2}$. Then we have:
    \begin{align}
        \F(p^\pi) &= \alpha_1 \EV_{p_1} [\log p_1 - \log p ^\pi] + \alpha_2 \EV_{p_2} [\log p_2 - \log p^\pi]\\
        &= \alpha_1 \EV_{p_1} \log p_1 + \alpha_2 \EV_{p_2} \log p_2 - \left(\alpha_1\EV_{p_1} \log p^\pi + \alpha_2 \EV_{p_2} \log^{\pi} \right) \label{eq:first_part_step}
    \end{align}
    We now write the following, where $\Bar{p}$ denotes $\Bar{p}^{pre}_1$:
    \begin{align}
        \EV_{\Bar{p}} \log p^\pi &= \int \log p^\pi(x) \Bar{p}(x) \; \der x\\
        &= \int \log p^\pi(x) \left[  \frac{\alpha_1 p_1}{\alpha_1 + \alpha_2} + \frac{\alpha_2 p_2}{\alpha_1 + \alpha_2} \right](x) \; \der x\\
        &= \frac{1}{\alpha_1 + \alpha_2} \left( \log p^\pi (x) \alpha_1 p_1(x) + \log p^\pi(x) \alpha_2 p_2(x)\right)\\
        &= \frac{1}{\alpha_1 + \alpha_2} \left( \alpha_1 \EV_{p_1} \log p^\pi + \alpha_2 \EV_{p_2} \log p^\pi \right) \label{eq:second_part_step}
    \end{align}
    By combining  Eq. \ref{eq:first_part_step} and \ref{eq:second_part_step}, we obtain:
    \begin{align}
        \F(p^\pi) &= \alpha_1 \EV_{p_1} \log p_1 + \alpha_2 \EV_{p2} \log p_2 - (\alpha_1 + \alpha_2) \EV_{\Bar{p}}\log p^\pi
    \end{align}
    Therefore,
    \begin{align}
        \argmin_\pi \F(p^\pi) &= \argmin_\pi \underbrace{\alpha_1 \EV_{p_1} \log p_1 + \alpha_2 \EV_{p_2} \log p_2}_{\text{constant}} - (\alpha_1 + \alpha_2) \EV_{\Bar{p}} \log p^\pi \\
        &= \argmin_\pi  - (\alpha_1 + \alpha_2) \EV_{\Bar{p}} \log p^\pi \\
        &= \argmin_\pi - (\alpha_1 + \alpha_2) \EV_{\Bar{p}} \log p^\pi + \underbrace{(\alpha_1 + \alpha_2) \EV_{\Bar{p}} \log \Bar{p}}_{\text{constant}}\\
        &= \argmin_\pi (\alpha_1 + \alpha_2) D_{KL}(\Bar{p} \| p^\pi)\\
    \end{align}
    Which concludes the proof.
\end{proof}

%% file: appendices/alg_implementation.tex
In the following, we provide an example of detailed implementations for \RewardGuidedFineTuningSolver employed in Sec. \ref{sec:algorithm} by \AlgNameLong, as well as \RewardGuidedFineTuningSolverRunningCosts, leveraged in Sec. \ref{sec:running_cost_KL} to scalably implement the AND operator. While the oracle implementation we report for completeness for \RewardGuidedFineTuningSolver corresponds to classic Adjoint Matching (AM) ~\citep{domingo2024adjoint}, the one for \RewardGuidedFineTuningSolverRunningCosts trivially extends AM base implementation to account for the running cost terms introduced in Eq. \ref{eq:running_cost_term}.

\subsection{Implementation of \RewardGuidedFineTuningSolver}

Before detailing the implementations, we briefly fix notation. Both algorithms explicitly rely on the interpolant schedules $\kappa_t$ and $\omega_t$ from \eqref{eq:flow_diff_eq}. In the flow-model literature, these are more commonly denoted $\alpha_t$ and $\beta_t$. We write $u^\textrm{pre}$ for the velocity field induced by the pre-trained policy $\pi^\textrm{pre}$, and $u^\textrm{fine}$ for the velocity field induced by the fine-tuned policy. In essence, each algorithm first draws trajectories and then uses them to approximate the solution of a surrogate ODE; its marginals serve as regression targets for the control policy~\citep[Section 5][]{domingo2024adjoint}.

\begin{algorithm}[H]
\caption{\RewardGuidedFineTuningSolverRunningCosts via AM}
\begin{algorithmic}[1]
\Require Pre-trained FM velocity field $u^{\text{pre}}$, step size $h$, number of fine-tuning iterations $N$, {gradient of reward $\nabla r$}, fine-tuning strength $\eta_k$
\State Initialize fine-tuned vector fields: $u^{\text{finetune}} = u^{\text{pre}}$ with parameters $\theta$.
\For{$n \in \{0, \ldots, N-1\}$}
  \State Sample $m$ trajectories $\bm{X} = (X_t)_{t \in \{0,\ldots,1\}}$ with memoryless noise schedule:
  \begin{equation}
      \sigma(t) = \sqrt{2 \kappa_t \!\left(\tfrac{\dot\omega_t}{\omega_t}\kappa_t - \dot\kappa_t\right)}
  \end{equation}
  \State{i.e.,:}
  \begin{equation}
    X_{t+h} = X_t + h\Big( 2 u_\theta^{\text{finetune}}(X_t, t) - \tfrac{\dot\omega_t}{\omega_t}X_t \Big) 
      + \sqrt{h}\,\sigma(t)\,\varepsilon_t, 
      \quad \varepsilon_t \sim \mathcal{N}(0,I), \quad X_0 \sim \mathcal{N}(0,I). \tag{51}
  \end{equation}
  \State For each trajectory, solve the \emph{lean adjoint ODE} backwards in time from $t=1$ to $0$, e.g.:
  \begin{equation}
    \tilde{a}_{t-h} = \tilde{a}_t + h\,\tilde{a}_t^\top \nabla_{X_t} 
    \Big( 2 v^{\text{base}}(X_t, t) - \tfrac{\dot\omega_t}{\omega_t}X_t \Big),
    \quad \tilde{a}_1 = \eta_k \nabla r(X_1). \tag{52}
  \end{equation}
  \State Note that $X_t$ and $\tilde{a}_t$ should be computed without gradients, i.e.,
    \begin{align}
        X_t &= \texttt{stopgrad}(X_t)\\
       \tilde{a}_t &= \texttt{stopgrad}(\tilde{a}_t)
    \end{align}
  \State For each trajectory, compute the following Adjoint Matching objective:
  \begin{equation}
    \mathcal{L}_{\text{Adj-Match}}(\theta) 
      = \sum_{t \in \{0,\ldots,1-h\}} 
      \left\| \tfrac{2}{\sigma(t)} \Big(v_\theta^{\text{finetune}}(X_t, t) - u^{\text{base}}(X_t, t)\Big) 
        + \sigma(t)\,\tilde{a}_t \right\|^2. \tag{53}
  \end{equation}
  \State Compute the gradient $\nabla_\theta \mathcal{L}(\theta)$ and update $\theta$ using favorite gradient descent algorithm.
\EndFor
\label{alg:noised_space_expansion}
\end{algorithmic}
\textbf{Output:} Fine-tuned vector field $v^{\text{finetune}}$
\end{algorithm}

\subsection{Implementation of \RewardGuidedFineTuningSolverRunningCosts}
The following \RewardGuidedFineTuningSolverRunningCosts is algorithmically identical to \RewardGuidedFineTuningSolverRunningCosts, with the only difference that the lean adjoint computation now integrates a running-cost term $f_t$, defined as follows (see Sec. \ref{sec:running_cost_KL}):
\begin{equation}
    f_t(x) \coloneqq \ \delta \left(\sum_{i=1}^n \alpha_i \, D_{KL}(p_t^\pi \,\|\, p_t^{pre,i})\right)(x,t), \quad t \in [0,1)
\end{equation}

\begin{algorithm}[H]
\caption{\RewardGuidedFineTuningSolverRunningCosts via AM with running costs}
\begin{algorithmic}[1]
\Require Pre-trained FM velocity field $v^{\text{base}}$, step size $h$, number of fine-tuning iterations $N$,  $f_t = \nabla  \delta \mathcal{G}_t(p^{\pi^k}_t)$, weight $\gamma_k$, weight schedule $\lambda$
\State Initialize fine-tuned vector fields: $v^{\text{finetune}} = v^{\text{base}}$ with parameters $\theta$.
\For{$n \in \{0, \ldots, N-1\}$}
  \State Sample $m$ trajectories $\bm{X} = (X_t)_{t \in \{0,\ldots,1\}}$ with memoryless noise schedule:
  \begin{equation}
      \sigma(t) = \sqrt{2 \kappa_t \!\left(\tfrac{\dot\omega_t}{\omega_t}\kappa_t - \dot\kappa_t\right)}
  \end{equation}
  \State{i.e.,:}
  \begin{equation}
    X_{t+h} = X_t + h\Big( 2 v_\theta^{\text{finetune}}(X_t, t) - \tfrac{\dot\omega_t}{\omega_t}X_t \Big) 
      + \sqrt{h}\,\sigma(t)\,\varepsilon_t, 
      \quad \varepsilon_t \sim \mathcal{N}(0,I), \quad X_0 \sim \mathcal{N}(0,I). \tag{40}
  \end{equation}
  \State For each trajectory, solve the \emph{lean adjoint ODE} backwards in time from $t=1$ to $0$, e.g.:
  \begin{align}
    \tilde{a}_{t-h} &= \tilde{a}_t + h\,\tilde{a}_t^\top \nabla_{X_t} 
    \Big( 2 v^{\text{base}}(X_t, t) - \tfrac{\dot\omega_t}{\omega_t}X_t  
    \Big)  - h\gamma_k \lambda_t f_t(X_t)\\
    \tilde{a}_1 &= -\gamma_k \lambda_1\nabla_{X_1} \delta\mathcal{G}_{1}(p^{\pi^k}_1)(X_1). \tag{41}
  \end{align}
  \State Note that $X_t$ and $\tilde{a}_t$ should be computed without gradients, i.e.,
  \begin{align}
        X_t &= \texttt{stopgrad}(X_t)\\
       \tilde{a}_t &= \texttt{stopgrad}(\tilde{a}_t)
    \end{align}
  \State For each trajectory, compute the Adjoint Matching objective:
  \begin{equation}
    \mathcal{L}_{\text{Adj-Match}}(\theta) 
      = \sum_{t \in \{0,\ldots,1-h\}} 
      \left\| \tfrac{2}{\sigma(t)} \Big(v_\theta^{\text{finetune}}(X_t, t) - v^{\text{base}}(X_t, t)\Big) 
        + \sigma(t)\,\tilde{a}_t \right\|^2. \tag{}
  \end{equation}
  \State Compute the gradient $\nabla_\theta \mathcal{L}(\theta)$ and update $\theta$ using a gradient descent step
\EndFor
\label{alg:fine_tuning_oracle}
\end{algorithmic}
\textbf{Output:} Fine-tuned vector field $u^{\text{finetune}}$
\end{algorithm}

%% file: appendices/app_computational_cost.tex
Reward-Guided Flow Merging (\AlgNameShort, see Alg. 1) is a sequential fine-tuning scheme which, at each of the ($K$) outer iterations, calls a reward-guided fine-tuning oracle such as \RewardGuidedFineTuningSolver (see Apx. \ref{sec:app_alg_implementation}). In practice, each oracle call performs ($N$) gradient steps of Adjoint Matching (see Apx. \ref{sec:app_alg_implementation}). At first sight, this suggests that the computational complexity of \AlgNameShort scales linearly in $K$ with respect to a standard fine-tuning run with ($N$) steps. However, this worst-case view does not fully capture the practical computational cost. We highlight two observations.

\paragraph{Approximate fine-tuning oracle.}
First, RFM can operate reliably with a rather \emph{approximate fine-tuning oracle}, i.e., with relatively small values of ($N$). We evaluate this phenomenon by replicating the objective curve of Fig. \ref{fig:toy_top_d} with same parameters and setting, for three different configurations of $(K,N)$ that keep the total budget ($K \cdot N = 300$) fixed but vary the outer (\ie $K$) and inner (\ie $N$) iteration counts:
\begin{itemize}
    \item $K = 10,; N = 30$
    \item $K = 15,; N = 20$ (as in Fig. \ref{fig:toy_top_d})
    \item $K = 30,; N = 10$
\end{itemize}

\newlength{\imgwl}
\begingroup
  \captionsetup[subfigure]{aboveskip=1.7pt, belowskip=0pt}
\setlength{\imgw}{0.25\textwidth}
\setlength{\imgwl}{0.33\textwidth}
\begin{figure*}[h]
    \centering
    \begin{subfigure}{\imgwl}
      \centering
      \includegraphics[width=\textwidth]{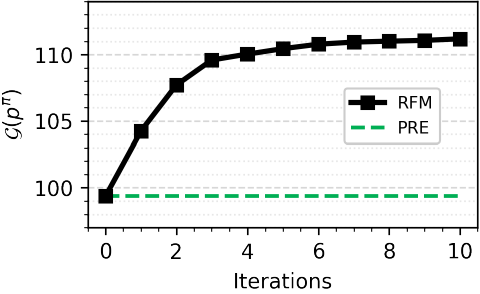}
      \caption{$K=10, N=30$}
      \label{fig:cost_a}
    \end{subfigure}%
    \begin{subfigure}{\imgwl}
      \centering
      \includegraphics[width=\textwidth]{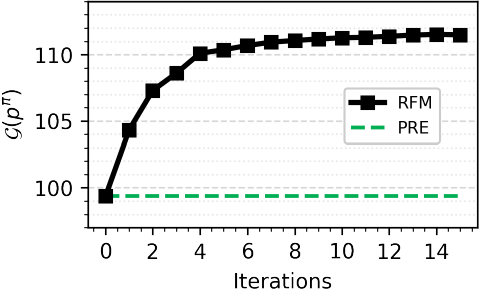}
      \caption{$K=15, N=20$}
      \label{fig:cost_b}
    \end{subfigure}%
    \begin{subfigure}{\imgwl}
      \centering
      \includegraphics[width=\textwidth]{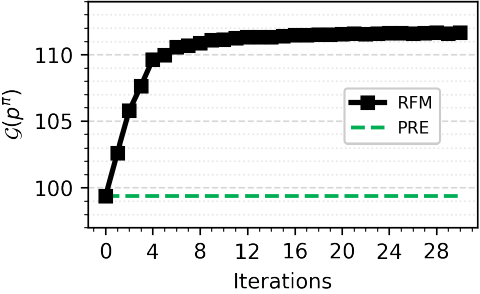}
      \caption{$K=30, N=10$}
      \label{fig:cost_c}
    \end{subfigure}%
    \vspace{0pt}
    \caption{\looseness-1 (left) \AlgNameShort run for reward-guided intersection with $K=10, N=30$, (center) \AlgNameShort run for reward-guided intersection with $K=15, N=20$, (right) \AlgNameShort run for reward-guided intersection with $K=30, N=10$.} \vspace{-0mm}
    \label{fig:experiments_cost}
\end{figure*}
\endgroup
\vspace{-0.5mm}

\looseness -1 The three corresponding curves are reported in Fig. \ref{fig:experiments_cost}. Empirically, all three settings achieve nearly identical final objective values, indicating that a more approximate oracle (smaller (N)) can be compensated by increasing the number of outer \AlgNameShort iterations ($K$), and vice versa, as long as the total optimization budget remains comparable. We observe a similar behaviour also on real-world, higher-dimensional, experiments (see Sec. \ref{sec:experiments} and Apx. \ref{sec:experimental_details}), where we values of $K$ vary from $K=1$ to $K=37$. 

\paragraph{$K/N$ Trade-off.}
Second, the runtimes of these configurations are of the same order. On our implementation, the runs with $((K,N) = (10,30), (15,20), (30,10))$ require approximately 1615 s, 1643 s, and 1870 s, respectively, showing a very light increase depending on $K$. This further supports the view that practitioners can trade off a cheaper but less accurate inner oracle (small ($N$)) against a slightly larger number of outer RFM steps (larger ($K$)), and vice versa, without incurring prohibitive additional cost. Since \AlgNameShort effectively solves a convex/non-convex optimization problem in probability space, we believe that classic convex optimization provides an interpretable framework for trading-off $N$ and $K$, by interpreting $N$ as the typical step-size, or learning rate, and $K$ as the typical number of gradient steps. Clearly, higher learning rates typically require less gradient steps and vice versa. Ultimately, one should notice that increasing $N$ does not directly imply better solution quality of the fine-tuning oracle, as it is the case for the oracle we employ within Sec. \ref{sec:experiments} (\ie Adjoint Matching \citep{domingo2024adjoint}), for which performance can degrade for excessively high values of $N$.

%% file: appendices/experimental_details.tex
\subsection{Illustrative Examples Experimental Details}
Numerical values in all plots shown within Sec. \ref{sec:experiments} are means computed over diverse runs of \AlgNameShort via $5$ different seeds. Error bars correspond to $95\%$ Confidence Intervals. 

\mypar{Shared experimental setup}
For all illustrative experiments we utilize Adjoint Matching (AM) [14 ]
for the entropy-regularized fine-tuning solver in Algorithm 1. Moreover, the stochastic gradient steps
within the AM scheme are performed via an Adam optimizer.

\mypar{Intersection Operator}
The balanced plot (see Fig. \ref{fig:toy_top_b} is obtained by running \AlgNameShort with $\alpha = [0.1, 0.1]$, for $K=80$ iterations, $\gamma_k = 28$, and $\lambda_t = 0.2$ for $t > 1 - 0.05$, and $\lambda_t = 0.4$ otherwise. 

For the balanced, reward-guided case in Fig. \ref{fig:toy_top_c}, we consider a reward function that is maximized by increasing the $x_2$ coordinate. We run \AlgNameShort with $\alpha = [0.1, 0.1]$, for $K=15$ iterations, $\gamma_k = 1.2$, and $\lambda_t = 0.2$ for $t > 1 - 0.05$, and $\lambda_t = 0.4$ otherwise. 

\mypar{Union Operator}

In both cases, we learn a critic via standard f-GAN~\citep{nowozin2016f} with $300$ gradient steps at each iteration $k \in [K]$ and continually fine-tune the same critic over subsequent iterations. For critic learning, we use a learning rate of $5\exp(-5)$.

For the balanced case, in Fig. \ref{fig:toy_mid_b}, we run \AlgNameShort with $\alpha = [1.0, 1.0]$. We use $K=13$ iterations, $\gamma_k = 0.001$.

For the unbalanced case in Fig. \ref{fig:toy_mid_c}, we run \AlgNameShort with $\alpha = [0.2, 1.8]$. Notice that up to normalization this is equivalent to $[0.1, 0.9]$ as reported in Fig. \ref{fig:toy_mid_c} for the sake of interpretability. We use $K=13$ iterations, $\gamma_k = 0.001$.

\mypar{Interpolation Operator}
In both cases, we learn a critic via standard f-GAN~\citep{nowozin2016f} with $800$ gradient steps at each iteration $k \in [K]$ and continually fine-tune the same critic over subsequent iterations. For critic learning, we use a learning rate of $1\exp(-5)$, and gradient penalty of $10.0$ to enforce $1$-Lip. of the learned critic.

For the case where $\pi^{init} \coloneqq \pi^{pre,1}$ (\ie left pre-trained model), in Fig. \ref{fig:toy_bottom_b}, we run \AlgNameShort with $\alpha = [1.0, 1.0]$. We use $K=6$ iterations, $\gamma_k = 1.0$.

For the case where $\pi^{init} \coloneqq \pi^{pre,2}$ (\ie right pre-trained model), in Fig. \ref{fig:toy_bottom_c}, we run \AlgNameShort with $\alpha = [1.0, 1.0]$. We use $K=6$ iterations, $\gamma_k = 1.0$.

\mypar{Complex Logic Expressions via Generative Circuits}
Pre-trained flows $\pi_1$ and $\pi_2$, as well as $\pi_1$ and $\pi_2$ are intersected via \AlgNameShort with $\gamma_k = 1$, for $K=20$, and $\lambda_t = 0.1$. The union operator is implemented with $K=30$, $\gamma_k = 0.0009$, $300$ critic steps and learning rate $5\exp(-5)$.

\subsection{Molecular Design Case Study}
Our base model FlowMol2 CTMC (\ie PRE-1) ~\citep{dunn2024mixedcontinuouscategoricalflow} is pretrained on the GEOM-Drugs dataset ~\citep{axelrod2022geom}.
We obtain our second model (\ie PRE-2) by finetuning PRE-1 with AM~\citep{domingo2024adjoint} to generate poses with lower single point total energy wrt. the continuous atomic positions as calculated with dxtb at the GFN1-xTB level of theory \cite{friede2024dxtb}. We then run \AlgNameShort with $K=50$, $\gamma=0.001$ for the balanced flow merging, and $K=20$, $\gamma=0.005$ to obtain the unbalanced flow merging. For reward-guided flow merging (RFM-RG), we set $\gamma=0.1$ and obtain the best model after $K=11$. All models start from PRE-1, \ie $\pi^{init}=\pi^{pre,1}$. All results for merging pre-trained models on GEOM can be found in Table \ref{tab:mol_exp_stats}. 
Running RFM-RG with $\alpha=3$ and $\gamma=0.001$, we obtain a model after $K=35$ that keeps the validity of its base models while implementing the reward-guided intersection.
We note that beyond validity, a critical step towards practical application will be to integrate molecular stability and synthesizability. Our \AlgNameShort formulation straightforwardly supports these extensions in the reward functional, and we leave their implementation to future work. 
\begin{table}[h]
    \centering
    \vspace{0.2cm}
    
    \begin{tabular}{lcc}
        \toprule
        & \textbf{Mean total energy} & \textbf{Mean validity} \\
        \textbf{Model} & \small{[Ha]} & \small{[\%]} \\
        \midrule
        
        PRE-1 & $-8.09 \pm 0.31$ & $76.44 \pm 1.7$ \\
        
        RFM-B & $-10.95 \pm 0.28$ & $74.34 \pm 0.9$ \\
        
        RFM-RG  & $-12.85 \pm 0.16$ &  $74.02 \pm 1.18$ \\
        
        RFM-UB & $-13.69 \pm 0.28$ & $72.78 \pm 0.4$ \\
        
        PRE-2 & $-14.76 \pm 0.29$ & $68.04 \pm 0.8$ \\
        
        \bottomrule
    \end{tabular}
    \caption{\label{tab:mol_exp_stats} Mean total energy and validity with standard deviation, averaged over 5 different seeds. \\ 
    \small{Suffixes: B - balanced ; UB - unbalanced; RG - reward-guided flow merging}}
\end{table}
For our second case-study - the OR operator - we use FlowMol2 CTMC pre-trained on QM9 \citep{ramakrishnan2014quantum}. 

We limit dimensionality to reduce the problem complexity by sampling 10 atoms per molecule, and run \AlgNameShort with $\gamma=100, K=37$. In particular Figure \ref{fig:union_qm9} shows that the estimated mean of the model $\pi^*$ obtained via \AlgNameShort matches the average total energy of $\pi^{pre,1}$ and $\pi^{pre,2}$ as predicted by the closed-form solution for the union operator presented in Sec. \ref{sec:problem_setting}. In Fig. \ref{fig:union_qm9}, OR denotes the final policy $\pi^*$ returned by \AlgNameShort.
\begin{figure}[H]
  \centering
  \includegraphics[width=0.3\textwidth]{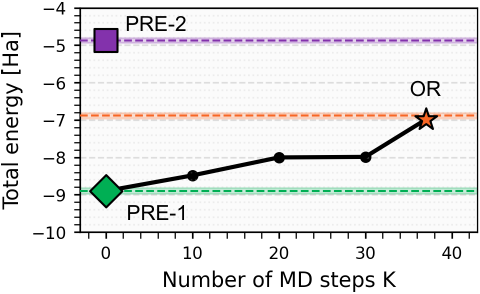}
  \caption{\looseness -1 Union on QM9}
  \label{fig:union_qm9}
\end{figure}

\subsection{Conformer Generation Case Study}
\label{apx:conformer}
We finetune the GEOM-QM9 pre-trained ETFlow model (denoted PRE-1) with AM on the molecular system \texttt{C\#C[C@H](C=O)CCC} to obtain PRE-2, using the same total energy objective as in the molecular design case study.
This is also the molecular system we perform our evaluations on.
For the subsequent merging experiments, we choose the lower-energy PRE-2 as the base model, \ie $\pi^{init}=\pi^{pre,2}$.
Balanced merging is performed with $\alpha_1=\alpha_2=1$, $\gamma=0.025$ and $K=6$.
The unbalanced merging is run with $\alpha_1=0.7$ and $\alpha_2=0.3$ and we take the model after $K=8$ steps with $\gamma=5e-5$.
The reward-guided merging model was obtained with $\gamma = 0.025$ after $K=6$, and the union model after $K=1$ with $\gamma = 1e-3$ and critics with the same GNN backbone as ETFlow. We show all results for the conformer generation case study in Tab.  \ref{tab:model_errors_conf}

\begin{table}[h]{
    \centering
    \vspace{0.2cm}
    \begin{tabular}{lcccc}
        \toprule
        & \textbf{$E$} & \textbf{$\mu$} & \textbf{$\Delta \epsilon$} & \textbf{$E_{min}$} \\
        \textbf{Model} & \small{[kcal/mol]} & \small{[debye]} & \small{[kcal/mol]} & \small{[kcal/mol]} \\
        \midrule
        
        PRE-1       & $0.3385 \pm 0.0002$ & $0.1679 \pm 0.0002$ & $0.5373 \pm 0.0019$ & $0.2793$ \\
        
        RFM-UB      & $0.3412$            & $0.1512$            & $0.5173$            & $0.2778$ \\
        
        RFM-B       & $0.3356 \pm 0.0001$ & $0.1503 \pm 0.0002$ & $0.4915 \pm 0.0014$ & $0.2782$ \\
        
        RFM-UNION   & $0.3352$            & $0.1467$            & $0.5033$            & $0.2761$ \\
        
        RFM-RG      & $0.3193 \pm 0.0003$ & $0.1141 \pm 0.0002$ & $0.4849 \pm 0.0015$ & $0.2777 \pm 0.0008$ \\
        
        PRE-2       & $0.3175 \pm 0.0006$ & $0.1268 \pm 0.0006$ & $0.4819 \pm 0.0010$ & $0.2761 \pm 0.0027$ \\
        
        \bottomrule
    
    \end{tabular}
    \caption{\label{tab:model_errors_conf}{Median Absolute Errors for energy $E$, dipole moment $\mu$, HOMO-LUMO gap $\Delta \epsilon$, and minimum energy $E_{min}$ across different models. We report mean and standard deviation over 5 different seeds.}}
    }
\end{table}

\section{Beyond Molecules: Reward-Guided Flow Merging of Pre-Trained Image Models}
\label{sec:image_experiments}
We further showcase the capabilities of \AlgNameLong on a small-scale, yet informative experiment for image generation. 
In the following, we consider pretrained CIFAR-10 image models~\citep{krizhevsky2009learning} and use the LAION aesthetics predictor V1~\citep{schuhmann2022laion} as a reward model.
Specifically, the aesthetics predictor was trained on a subset of the SAC dataset~\citep{SAC_data} with available ratings from $1$ (low preference / aesthetics) to $10$ (high preference). 
The goal of this case study is to show that \AlgNameShort can merge two models, PRE-1 and PRE-2, while optimizing the aesthetics score.
We perform reward-guided flow merging with PRE-2 as the base model, obtaining the model RFM-RG after $K=11$ iterations with $\gamma=1$ and $\alpha_i=1$. 
The numerical results in Tab. \ref{tab:image_exp_stats} show that \AlgNameShort can successfully intersect multiple prior flow image models while maximizing the aesthetic score. In particular, the fine-tuned model achieves a score of $3.64 \pm 0.53$ against $3.16 \pm 0.66$ and $3.23 \pm 0.58$ of PRE-1 and PRE-2 respectively.
We also report sample images of the discussed models in Fig. \ref{fig:cifar_images}.

\begin{table}[ht]
{
    \centering
    \vspace{0.2cm}
    
    \begin{tabular}{lc}
        \toprule
        \textbf{Model} & \textbf{Mean aesthetic score} \\ 
        \bottomrule

        PRE-1 & $ 3.16\pm 0.66$ \\

        PRE-2 & $ 3.23\pm 0.58$ \\ 

        RFM-RG & $3.64 \pm 0.53$ \\ 

        \bottomrule
    \end{tabular}
    \caption{\label{tab:image_exp_stats} \AlgNameShort can perform reward-guided (RG) intersections of pre-trained CIFAR-10 image models~\citep{krizhevsky2009learning}. We evaluate the resulting models in terms of mean aesthetic score (\ie the reward) over $1000$ samples, and report one std.}}
\end{table}

\begin{figure}[ht]
    \centering
    
    \begin{subfigure}{\linewidth}
        \centering
        \includegraphics[width=0.5\linewidth]{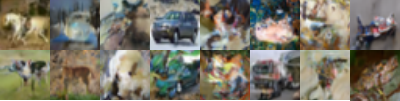}
        \caption{PRE-1}
    \end{subfigure}

    \vspace{0.2cm}
    
    \begin{subfigure}{\linewidth}
        \centering
        \includegraphics[width=0.5\linewidth]{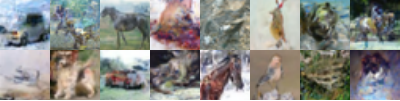}
        \caption{PRE-2}
    \end{subfigure}

    \vspace{0.2cm}

    \begin{subfigure}{\linewidth}
        \centering
        \includegraphics[width=0.5\linewidth]{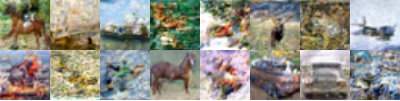}
        \caption{RFM-RG}
    \end{subfigure}

    \caption{Images generated by the two pre-trained flow models (\ie PRE-1, PRE-2), and by the flow model obtained via reward-guided intersection (\ie RFM-RG).}
    \label{fig:cifar_images}
\end{figure}